\documentclass[twoside,11pt]{article}

\usepackage{jmlr2e}

\RequirePackage{amsmath,amsmath,amsfonts,amssymb}
\RequirePackage{natbib}

\usepackage{color}

\usepackage{graphicx}
\usepackage{caption}
\usepackage{subcaption}

\renewcommand{\hat}{\widehat}
\renewcommand{\tilde}{\widetilde}

\newcommand{\sd}{\mathfrak d}

\DeclareMathOperator*{\argminnew}{arg\,min}

\newcommand{\vct}{\boldsymbol }

\newcommand{\ud}{\mathrm d}

\newcommand{\kl}{\mathrm{KL}}

\newcommand{\SPAN}{\mathrm{span}}

\newcommand{\diag}{\mathrm{diag}}

\newcommand{\op}{\mathrm{op}}

\newcommand{\cavg}{\mathsf{CA}}
\newcommand{\cw}{\mathsf{CW}}
\newcommand{\rnn}{\mathsf{R}}

\newcommand{\sP}{\mathsf{Q}}
\newcommand{\sS}{\mathsf{S}}

\newcommand{\avg}{\mathrm{avg}}

\ShortHeadings{sample-complexity of Estimating Convolutional and Recurrent Neural Networks}{Du, Wang et al.}
\firstpageno{1}

\begin{document}

\title{How Many Samples are Needed to Estimate a Convolutional or Recurrent Neural Network?
\thanks{A preliminary version of this paper titled ``How Many Samples are Needed to Estimate a Convolutional Neural Network''
appeared in Proceedings of the 32nd Conference on Neural Information Processing Systems (NeurIPS 2018), with results for convolutional neural networks only.}}

\author{\name Simon Du\thanks{Simon Du and Yining Wang contributed equally to this work.} \email ssdu@cs.cmu.edu \\
       \name Yining Wang\textsuperscript{$\dagger$} \email yiningwa@cs.cmu.edu\\
       \addr Machine Learning Department, School of Computer Science\\
       Carnegie Mellon University, Pittsburgh, PA 15213, USA\\
       \name Xiyu Zhai \email xiyuzhai@mit.edu\\
       \addr Department of Electrical Engineering and Computer Science\\
       Massachusetts Institute of Technology, Cambridge, MA 02139, USA\\
       \name Sivaraman Balakrishnan \email siva@stat.cmu.edu\\
       \addr Department of Statistics and Data Science\\
       Carnegie Mellon University, Pittsburgh, PA 15213, USA\\
       \name Ruslan Salakhutdinov \email rsalakhu@cs.cmu.edu\\
       \name Aarti Singh \email aarti@cs.cmu.edu\\
       \addr Machine Learning Department, School of Computer Science\\
       Carnegie Mellon University, Pittsburgh, PA 15213, USA
       }

\editor{}

\maketitle

\begin{abstract}
It is widely believed that the practical success of Convolutional Neural Networks (CNNs) and Recurrent Neural Networks (RNNs)
owes to the fact that CNNs and RNNs use a more compact parametric representation than their Fully-Connected Neural Network (FNN) counterparts, and consequently require fewer training examples to accurately estimate their parameters.
We initiate the study of rigorously characterizing the sample-complexity of estimating CNNs and RNNs.
We  show that the sample-complexity to learn CNNs and RNNs scales linearly with their intrinsic dimension and this sample-complexity is much smaller than for their FNN counterparts.
For both CNNs and RNNs, we also present lower bounds showing our sample complexities are tight up to logarithmic factors.
Our main technical tools for deriving these results are a localized empirical process analysis and a new technical lemma characterizing the convolutional and recurrent structure. We believe that these tools may inspire further developments in understanding CNNs and RNNs.
\end{abstract}


\maketitle






\section{Introduction}
Convolutional Neural Networks (CNNs) and  Recurrent Neural Networks (RNNs) have achieved remarkable impact in many machine learning applications.
The key building block of these improvements is the use of weight sharing layers to replace traditional fully connected layers, dating back to~\cite{lecun1995convolutional,rumelhart1988learning}.
A common folklore for explaining the success of CNNs and RNNs is that they use a more compact  representation than Fully-connected Neural Networks (FNNs) and thus require fewer samples to reliably estimate.
However, to our knowledge, there is no rigorous characterization of the precise sample-complexity of learning a CNN or an RNN and thus it is unclear, from a statistical point of view, why using CNNs or RNNs often results in 
a better performance than just using FNNs.

\noindent {\bf Our Contributions: } In this paper, we take a step towards understanding the statistical behavior of CNNs and RNNs.
We adapt tools from localized empirical process theory~\citep{geer2000empirical} and combine them with a structural property of convolutional filters in CNNs (see Lemma~\ref{lem:covering-cavg},~\ref{lem:covering-cw}) or the recurrent transition matrix in RNNs (see Lemma~\ref{lem:covering-rnn}) to give a sharp characterization of the sample-complexity of estimating simple CNNs and RNNs. 
\begin{enumerate}
\item We first consider the problem of estimating a convolutional filter with \emph{average} pooling (described in~Section~\ref{sec:model-specification}) using the least squares estimator.
We show in the standard statistical learning setting, under some conditions on the input distribution, the least squares estimate
$\hat{w}$ satisfies:
\[\sqrt{\mathbb{E}_{x \sim \mu}|F^{\cavg}(x,\hat{w})-F^{\cavg}(x,w_0)|^2}= \widetilde{O}\left(\sqrt{{m}/{n}}\right),\] 
where $\mu$ is the input distribution, $w_0$ is the underlying true convolutional filter, $m$ is the filter size,
and $F^{\cavg}(\cdot)$ denotes the convolutional network with average pooling.
Notably, to achieve an $\epsilon$ error, the CNN only needs $\widetilde{O}({m}/{\epsilon^2})$ samples whereas the FNN needs $\Omega({d}/{\epsilon^2})$ with $d$ being the input size.
Since the filter size $m \ll d$, this result clearly justifies the folklore that the convolutional layer is a more compact representation.
Furthermore, we complement this upper bound with a minimax lower bound which shows the error bound $\widetilde{O}(\sqrt{{m}/{n}})$ is tight up to logarithmic factors.

\item Next, we consider a one-hidden-layer CNN in which the filter $w \in \mathbb{R}^{m}$ and output weights $a \in \mathbb{R}^{r}$ are unknown.
%
This architecture was previously considered in~\cite{du2017spurious}.
However, the focus of that work was on understanding the dynamics of gradient descent.
Using similar tools as in analyzing a single convolutional filter, we show that the least squares estimator achieves the error bound $\widetilde{O}(\sqrt{{(m+r)}/{n}})$ if the ratio between the stride size and the filter size is a constant.
Further, we present a minimax lower bound showing that the obtained rate is tight up to logarithmic-factors.

\item Lastly, we consider an RNN as described in~\eqref{eq:rnn-model}. 
Based on a new structural lemma for the RNN model, we show that the least squares estimator has prediction error upper bounded as $\widetilde{O}(\sqrt{{dr}/{n}})$, where $d$ is the input dimension and $r$ is the dimension of the hidden state.
On the other hand, the corresponding FNN has $Ld$ features where $L$ is the length of the input sequence.
In typical applications, we have that $r\ll L \ll d$~(see for instance the paper of \cite{mikolov2010recurrent}).
Our result demonstrates the sample-complexity benefits from using the RNN to exploit the hidden structure rather than using the FNN.
\end{enumerate}
To our knowledge, these theoretical results are the first sharp analyses of the statistical sample-complexity of the CNN and RNN.

\subsection{Comparison with existing work}
\label{sec:rel}
Our work is closely related to the analysis of the generalization ability of neural networks~\citep{arora2018stronger,anthony2009neural,bartlett2017nearly,bartlett2017spectrally,neyshabur2017pac,konstantinos2017pac,li2018tighter}.
These generalization bounds are often of the form:
\begin{align}
L(\theta) - L_{\text{tr}}(\theta) \le D/\sqrt{n}
\label{eq:generalization}
\end{align} 
where $\theta$ represents the parameters of a neural network, $L(\cdot)$ and $L_{\text{tr}}(\cdot)$ represent population and empirical error under some \emph{additive} loss,
and $D$ is the model capacity and is finite only if the (spectral) norm of the weight matrix for 
each layer is bounded.
Comparing with generalization bounds based on model capacity, our result has two advantages: 
\begin{itemize}
	\item If $L(\cdot)$ is taken to be the mean-squared\footnote{Because the mean-squared error $\mathbb{E}|\cdot|^2$ is a sum of independent random variables, it is common to apply generalization error bounds directly on this quantity.}
	Eq.~(\ref{eq:generalization}) implies an $\tilde O(1/\epsilon^4)$ sample-complexity to achieve a standardized mean-square error of $\sqrt{\mathbb E|\cdot|^2}\leq \epsilon$,
	which is considerably larger than the $\tilde O(1/\epsilon^2)$ sample-complexity we establish in this paper.
	\item Since the complexity $D$ of a model class in regression problems typically depends on the magnitude of model parameters,
	generalization error bounds like~\eqref{eq:generalization} are not scale-independent and deteriorate if the magnitude of the parameters is large.
	In contrast, our analysis has no dependence on the magnitude.
\end{itemize}
On the other hand, we consider the special case where the neural network model is well-specified and the labels are generated according to a neural network with unbiased additive noise (see~\eqref{eq:model-general}) whereas the generalization bounds discussed in this section are typically model agnostic.

\subsection{Other related work}
\label{sec:other_rel}
Recently, researchers have made progress in theoretically understanding various aspects of neural networks, including understanding the hardness of estimation~\citep{goel2016reliably,song2017complexity,brutzkus2017globally}, the landscape of the loss function~\citep{kawaguchi2016deep,choromanska2015loss,hardt2016identity,haeffele2015global,freeman2016topology,safran2016quality,zhou2017landscape,nguyen2017loss,nguyen2017loss2,ge2017learning,zhou2017landscape,safran2017spurious,du2018power}, the dynamics of gradient descent~\citep{tian2017analytical,zhong2017recovery,li2017convergence}, and developing provable learning algorithms~\citep{goel2017eigenvalue,goel2017learning,zhang2015learning}.

Focusing on the convolutional neural network, most existing work has analyzed the convergence rate of gradient descent or its variants~\citep{du2017convolutional,du2017spurious,goel2018learning,brutzkus2017globally,zhong2017learning}.
Our paper differs from these past works in that we do not consider the computational-complexity but only the sample-complexity and the fundamental information theoretic limits of estimating a CNN.

%
The convolutional structure has also been studied in the dictionary learning~\citep{singh2018minimax} and blind de-convolution~\citep{zhang2017global} literature.
These papers studied the unsupervised setting where their goal is to recover structured signals from observations generated according to convolution operations whereas our paper focuses on the supervised learning setting where the target (ground-truth) predictor has a convolutional structure.

Our formulation of an RNN can be viewed as a special case of the classical 
\citep{kalman1960new}
problem of learning a linear dynamical system~\citep{hazan2017learning,hardt2018gradient,simchowitz2018learning,oymak2018non}.
These recent works consider both computational and statistical issues and to our knowledge, their sample-complexity results are not tight.

Lastly, a line of recent works has studied over-parameterized neural networks, requiring the width of the neural network at every layer to be larger than the number of data points~\citep{du2018provably,du2018global,allen2018convergence,allen2018convergence,allen2019can,zou2018stochastic,li2018learning,arora2019fine}.
In particular, \citet{arora2019fine,allen2018learning,allen2019can} showed that these over-parameterized neural networks can also generalize in some cases.
These works are different from ours in their focus. We do not consider the over-parameterized setup. Instead we present tight information-theoretic characterizations of the fundamental statistical limits of estimating CNNs and RNNs.

\section{Preliminaries}
In this section, we introduce the convolutional filter, convolutional neural network and recurrent neural network models that we study. 
We then introduce briefly the least squares estimator that we study for our upper bounds, and introduce the minimax risk which we subsequently lower bound.

\subsection{Problem Setup}\label{sec:model-specification}

\begin{figure*}[t!]
	\centering
	\begin{subfigure}[t]{0.45\textwidth}
		\includegraphics[width=\textwidth]{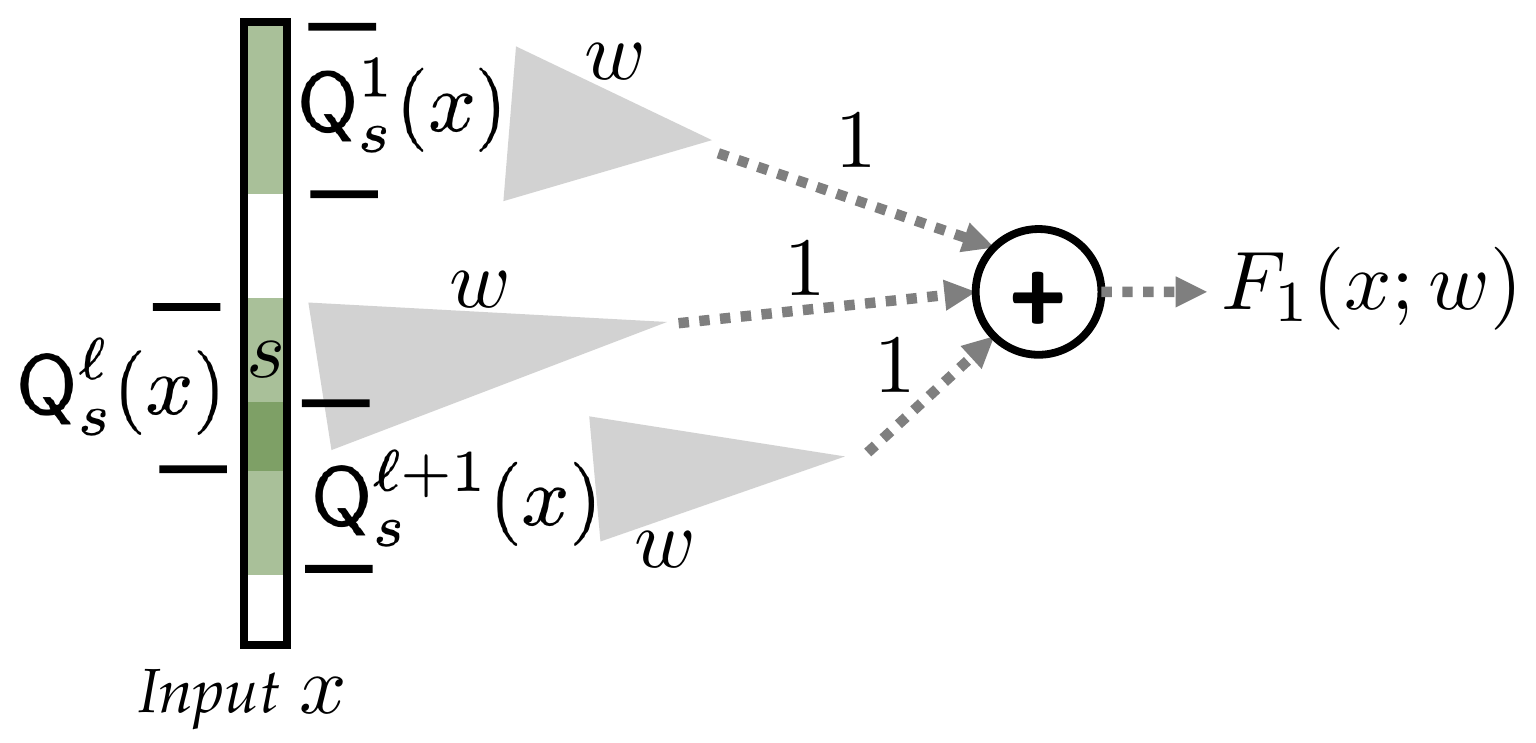}
		\caption{
		Prediction function formalized~\eqref{eqn:conv_filter_avg_pooling}.
		It consists of a convolutional filter followed by averaged pooling.
		The convolutional filter is unknown.
		}
		\label{fig:archi_filter}
	\end{subfigure}	
	\qquad
	\begin{subfigure}[t]{0.45\textwidth}
		\includegraphics[width=\textwidth]{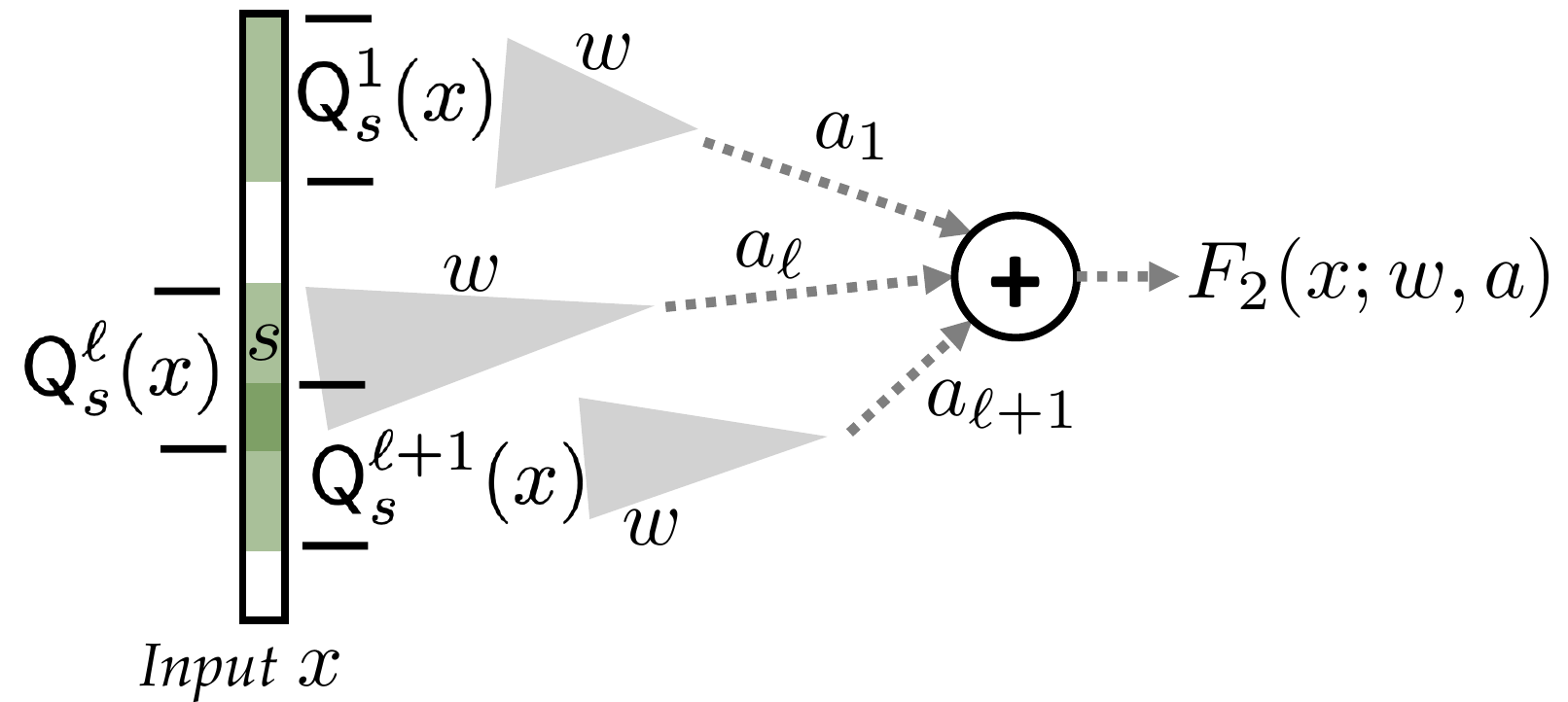}
		\caption{Prediction function formalized in~\eqref{eqn:two_layer}
		It consists of a convolutional filter followed by a linear prediction layer.
		Both layers are unknown.
		}
		\label{fig:archi_twolayer}
	\end{subfigure}	
	\quad
	\caption{CNN architectures that we consider in this paper. 
	}
	\label{fig:architecture}
\end{figure*}


The $i$-th labeled data point is denoted as $(X^i,Y^i)$, where $X^i$ is the input vector for the $i$-th data point and $Y_i\in\mathbb R$ represents its corresponding label.
The basic models we study in this paper are best abstracted in the form,
\begin{equation}
Y_i = F(X_i;\theta) + \xi_i,
\label{eq:model-general}
\end{equation}
where $F$ represents the network, $\theta\in\Theta$ are the underlying true parameters of the network, and 
$\{\xi_i\}$ are zero-mean random variables capturing the measurement noise.

\subsubsection{Convolutional neural networks with average pooling}
We consider convolutional neural networks (CNN) with vector inputs, represented by $x^i\in\mathbb R^d$.
The convolutional filter is assumed to be of size $m$, with weight vector $w\in\mathbb R^m$.
The filter is applied to different segments of the input vector $x^i$, with a \emph{stride} of $s$.
More specifically, the CNN computes the inner products of
\begin{equation}
w^\top \sP_s^0(x^i),\;\; w^\top\sP_s^1(x^i), \;\; \ldots, \;\; w^\top \sP_s^{\lfloor (d-m)/s\rfloor}(x^i),
\label{eq:cnn-inner-products}
\end{equation}
where $\sP_s^\ell(x^i)=(x_{\ell s+1}^i,\ldots,x_{\ell s+m}^i)$ is an $m$-dimensional segment of $x^i$.
Afterwards, \emph{average pooling} is used to aggregate the convolved inner products to obtain the final output:
\begin{equation}
Y^i = F^{\cavg}(X^i; w) +\xi_i = \sum_{\ell=0}^{\lfloor (d-m)/s\rfloor} w^\top \sP_s^\ell (x^i) + \xi_i.
\label{eqn:conv_filter_avg_pooling}
\end{equation}
A graphical illustration of the CNN model with average pooling is given in Fig.~\ref{fig:archi_filter}.
Throughout the remainder of the paper, in order 
simplify our analysis and notation we assume that both $d$ and $m$ are divisible by $s$, 
and consequently that $\lfloor (d-m)/s\rfloor=(d-m)/s$.

\subsubsection{Convolutional neural networks with weighted pooling}

In addition to CNNs with average pooling, we also consider CNNs with an additional unknown \emph{weighted} pooling layer,
making the model essentially a two-layer neural network.
To be more specific, building upon the convolutional inner products computed in Eq.~(\ref{eq:cnn-inner-products}),
the outputs of CNNs with weighted pooling can be modeled as
\begin{equation}
Y^i = F^{\cw}(X^i; w,a) +\xi_i = \sum_{\ell=0}^{\lfloor (d-m)/s\rfloor} a_\ell w^\top \sP_s^\ell (x^i) + \xi_i,
\label{eqn:two_layer}
\end{equation}
where $a=(a_0,a_1,\ldots,a_{\lfloor(d-m)/s\rfloor})\in\mathbb R^{\lfloor(d-m)/s\rfloor+1}$
is an unknown vector of the additional weighted pooling layer, and $\{\xi_i\}_{i=1}^n$ are noise variables.

A graphical illustration of the CNN model with weighted pooling is given in Fig.~\ref{fig:archi_twolayer}.
Again, we assume that both $d$ and $m$ are divisible by $s$, 
and therefore $\lfloor (d-m)/s\rfloor=(d-m)/s$.

\subsubsection{Recurrent neural networks}
The recurrent neural network (RNN) is assumed to have $r$ hidden units.
More specifically, each input element $x_t^i$ is associated with a latent representation $h_t^i\in\mathbb R^r$.
The network also has a pre-specified ``starting state'' $h^i = 0$.
The dynamics of the latent representation is modeled as: 
\begin{equation}
h_t^i = A h_{t-1}^i + B x_t^i, \;\;\;\;\;\; t=1, 2, \ldots, L,
\label{eq:rnn-transition}
\end{equation}
where $A\in\mathbb R^{r\times r}$ and $B\in\mathbb R^{r\times d}$ are unknown weight matrices to be learnt.
We assume a linear (identity) activation in~\eqref{eq:rnn-transition}. Finally, the regression response $Y^i$ corresponding to $X^i$ is modeled by an average pooling over the final state, i.e.:
\begin{equation}
Y^i = F^{\rnn}(X^i,A,B)+\xi_i = \vct 1^\top h_L^i + \xi_i,
\label{eq:rnn-model}
\end{equation}
where $\{\xi_i\}_{i=1}^n$ are noise variables.



\subsection{Least-squares estimation}

The estimators we consider throughout this paper are \emph{least-squares} estimators.
More specifically, on a training data set $\{X^i,Y^i\}_{i=1}^n$ generated from an underlying network $F$,
we solve the following problem to estimate the network parameters:
\begin{equation}
\hat\theta \in \argminnew_{\theta\in\Theta} \sum_{i=1}^n \big| Y^i - F(X^i;\theta)\big|^2.
\label{eq:ls}
\end{equation}
Note that this optimization problem might not have a unique solution
for networks $F^{\cw}$ or $F^{\rnn}$. For instance,
we may choose different scalings between $w$ and $a$ in $F^{\cw}$,
or exchange two hidden units in $F^{\rnn}$ when $r\geq 2$, and obtain the same function $F$ and objective value.
In such cases, any solution $\hat\theta$ leading to the optimal least-squares objective value can be chosen, and
our statistical guarantees apply to this estimate $\hat\theta$.

We remark that we only focus on the statistical rate of convergence for this estimator and we leave the analysis of computational complexity of solving the least squares problem as future work.
In our experiments, we simply use gradient descent to obtain an estimate. 


\subsection{Assumptions and minimax analysis}

We use \emph{minimax analysis} \citep{lehmann2006theory,tsybakov2009introduction,wasserman2013all} to understand the fundamental limit of estimating convolutional and recurrent neural networks.
We being by introducing two regularity assumptions imposed on the distributions of $\{X^i\}$ and $\{Y^i\}$:
\begin{enumerate}
\item[(A1)] (Sub-gaussian noise): $\{\xi_i\}_{i=1}^n$ are independent, centered sub-Gaussian random variables with sub-Gaussian parameters upper bounded by $\sigma^2$;
\item[(A2)] (Non-degenerate random design): there exists a centered underlying sub-Gaussian distribution $\mu$ with sub-Gaussian parameter $C$ over $\mathbb R^d$ such that $\{x^i\}_{i=1}^n$
 (for $F^{\cavg}$ and $F^{\cw}$) or $\{x_t^i\}_{i,t=1}^{n,L}$ (for $F^{\rnn}$) is i.i.d.~sampled from $P_X$;
 furthermore $c I \preceq \mathbb E_{\mu}[xx^\top] \preceq C I$ for some constants $0<c\leq C<\infty$.
\end{enumerate}
To resolve the issues of non-identifiability of the parameter $\theta$, we use the \emph{population prediction error} instead of the 
more classical parameter estimation error to characterize the information-theoretic sample-complexity of estimating convolutional or recurrent neural networks.
Given a parameter estimate $\hat\theta$ and the true underlying parameter $\theta$ of a neural network $F$,
the mean-square prediction error is defined as:
\begin{equation}
\mathrm{err}(\hat\theta,\theta) := \sqrt{\mathbb E_\mu|F(x;\theta)-F(x;\hat\theta)|^2},
\label{eq:err}
\end{equation}
where $\mu$ is the (unknown) underlying distribution defined in Assumption (A2).

To evaluate and benchmark the quality of the least-squares estimator, we adopt the minimax framework to characterize the fundamental hardness of the estimation problems we study in this paper.
The \emph{minimax risk} of prediction is defined as
\begin{equation}
\mathfrak M(n; F) := \inf_{\hat\theta}\sup_{\theta}\mathbb E_{\mu,\theta}\left[\mathrm{err}(\hat\theta,\theta)\right],
\label{eq:minimax}
\end{equation}
where the $\mathbb E_{\mu,\theta}$ notation summarizes the data generating process $\{X^i\}_{i=1}^n\overset{\text{i.i.d}}{\sim}\mu$, 
$Y^i=F(X^i;\theta)+\xi_i$,
and we use the notation $\mathfrak M(n;F)$ to emphasize that the minimax risk depends crucially on the specific type of networks to be estimated and the size of the training sample.
Although the minimax risk typically also depends on other problem parameters such as the input dimension, we suppress this dependency in the notation $\mathfrak M(n; F)$ to concisely present our results.


\section{Main results}
In this section, we present our main upper and lower bounds on the minimax risk.

\subsection{Upper bounds}\label{results-upper}


Throughout this section we assume the assumptions (A1) and (A2) hold.
We also suppress constants potentially depending on $c$ and $C$ (defined in Assumption A2) in the asymptotic $\lesssim$ notation.
We will establish the following upper bounds on the minimax prediction error of convolutional or recurrent neural networks.

\begin{theorem}
For $\delta\in(0,1/2)$ and sufficiently large\footnote{A detailed scalings of $n$ and other problem dependent parameters are given in the remarks immediately following Theorem \ref{thm:upper-bound}.} $n$,
with probability $1-\delta$ over the random draws of $\{x^i\}_{i=1}^n$ (for $F^{\cavg}$ and $F^{\cw}$) or $\{x_{t}^i\}_{i,t=1}^{n,L}$ (for $F^{\rnn}$), it holds that
\begin{eqnarray}
\mathfrak M(n; F^{\cavg}) &\lesssim& \sqrt{\frac{\sigma^2m\log d}{n}}; \label{eq:upper-bound-conv1}\\
\mathfrak M(n;F^{\cw}) &\lesssim& \sqrt{\frac{\sigma^2 \min\{d, m+(d/s)\times (m/s)\}\cdot \log d}{n}};\label{eq:upper-bound-conv2}\\
\mathfrak M(n; F^{\rnn}) &\lesssim& \sqrt{\frac{\sigma^2(d+L)\min\{r,d\}\log(Ld)}{n}}.
\label{eq:upper-bound-rnn}
\end{eqnarray}
\label{thm:upper-bound}
\end{theorem}

\noindent \textbf{Remarks}:
\begin{enumerate}
\item All upper bounds are conditioned on the random draws of $\{x^i\}_{i=1}^n$ or $\{x^i_t\}_{i,t=1}^{n,L}$ (i.e, with expectation taken over the randomness of the noise variables $\{\xi_i\}_{i=1}^n$),
and are attained by the least-squares estimator defined in~\eqref{eq:ls}.

\item The number of training data points $n$ is ``sufficiently large'' under the context of Theorem \ref{thm:upper-bound} if it satisfies:
\begin{eqnarray*}
\text{For $F^{\cavg}$}: & & n \gtrsim c^{-1}C^2m\cdot \log(c^{-1}C d\log(n/\delta))\log^2(n/\delta),\\
\text{For $F^{\cw}$}: && n \gtrsim  c^{-1}C^2 \min\{d,m+(m/s)\times (d/s)\}\cdot \log(c^{-1}Cd\log(n/\delta))\log^2(n/\delta),\\
\text{For $F^{\rnn}$:} && n \gtrsim c^{-1}C^2(d+L)\min\{d,r\}\cdot \log(c^{-1}C Ld\log(n/\delta))\log^2(n/\delta).
\end{eqnarray*}
\end{enumerate}
All upper bounds in Theorem \ref{thm:upper-bound} have convergence rates $O(1/\sqrt{n})$ for the standardized mean-square error (see Eq.~(\ref{eq:err})),
which are in contrast to previous works based on concentration inequalities
yielding mostly $1/n^{1/4}$ convergence rates~\citep{bartlett2017spectrally,neyshabur2015norm}.
Furthermore, all upper bounds in Theorem \ref{thm:upper-bound} are \emph{scale-invariant}, because they do not depend in any way 
on the magnitude of the network weights $w, a,A$ or $B$.

Omitting logarithmic terms, the results in Theorem \ref{thm:upper-bound} also match the intuition of ``parameter counts'', 
which simply counts the number of unknown weight parameters in a neural network.
More specifically, for the $F^{\cavg}$ network, there are $m$ unknown weight parameters ($w\in\mathbb R^m$),
which matches the $\tilde O(\sqrt{m/n})$ rate in Eq.~(\ref{eq:upper-bound-conv1});
for the $F^{\cw}$ network, there are $(m+d/s)$ unknown weight parameters ($w\in\mathbb R^m$ and $a\in\mathbb R^{d/s}$),
which matches the $\tilde O(\sqrt{(\min\{d,m+(m/s)\times(d/s)\})/n})$ rate in Eq.~(\ref{eq:upper-bound-conv2})
when the stride $s$ is on the same order of the filter size $m$ (i.e., $m/s=O(1)$);
for the $F^{\rnn}$ network, there are $(dr+r^2)$ unknown weight parameters ($A\in\mathbb R^{r\times r}$ and $B\in\mathbb R^{r\times d}$),
which matches the $\tilde O(\sqrt{((d+L)\min\{d,r\})/n})$ rate in Eq.~(\ref{eq:upper-bound-rnn})
when $r\ll L\ll d$, a common setting in natural language processing applications \citep{mikolov2010recurrent}.


The three upper bound results in Theorem \ref{thm:upper-bound} share the same proof framework,
yet with different covering number analysis tailored to each neural network structure separately.
The proof framework is built upon the probability tool of \emph{self-normalized empirical processes}, 
which (with high-probability) upper bounds the supremum of an empirical process with suitable normalization.
Such upper bounds would eventually depend on the covering numbers of self-normalized parameter spaces,
which we upper bound for different network structures ($F^{\cavg}$, $F^{\cw}$ and $F^{\rnn}$) separately.

\subsection{Lower bounds}\label{sec:results-lower}

To complement our results in Theorem \ref{thm:upper-bound}, 
we prove the following theorem which establishes \emph{lower bounds} on the minimax rates $\mathfrak M(n;\cdot)$,
showing the information-theoretic limits of sample-complexity that no estimator could violate.

\begin{theorem}
Suppose $\{x^i\}_{i=1}^n\overset{i.i.d}{\sim} \mathcal N(0,I)$ for $F^{\cavg}, F^{\cw}$ and $\{x_t^i\}_{i,t=1}^{n,L}
\overset{i.i.d.}{\sim} \mathcal N(0,1)$ for $F^{\rnn}$.
Suppose also $\{\xi_i\}_{i=1}^n\overset{i.i.d.}{\sim}\mathcal N(0,\sigma^2)$.
Then there exists a universal constant $C>0$ such that
\begin{eqnarray}
\mathfrak M(n;F^{\cavg}) &\geq& C\sqrt{\frac{\sigma^2 m}{n}};
\label{eq:lower-bound-conv1}\\
\mathfrak M(n;F^{\cw}) &\geq& C\sqrt{\frac{\sigma^2(m+d/s)}{n}};\label{eq:lower-bound-conv2}\\
\mathfrak M(n;F^{\rnn}) &\geq& C\sqrt{\frac{\sigma^2 \min\{rd,Ld\}}{n}}.\label{eq:lower-bound-rnn}
\end{eqnarray}
\label{thm:lower-bound}
\end{theorem}
\noindent \textbf{Remarks}:
\begin{enumerate}
\item Theorem \ref{thm:lower-bound} establishes lower bounds for the \emph{worst-case} prediction error of $F^{\cavg},F^{\cw}$ and $F^{\rnn}$
for \emph{any} learning algorithm that takes as input $n$ labeled data points and outputs a prediction network $\hat F$.
This is clear from the definition of the minimax rates $\mathfrak M(n;\cdot)$.
\item While Theorem \ref{thm:lower-bound} considers isotropic Gaussian $\{x^i\}$, $\{x^i_t\}$
and Gaussian noises $\{\xi_i\}$, 
this should \emph{not} be interpreted as a limitation because the isotropic Gaussian data points and noises are a special case of the general learning problem
covered by the upper bounds in Theorem \ref{thm:upper-bound}. 
Hence, a lower bound for the isotropic Gaussian case implies a lower bound for the more general case.
\end{enumerate}

The lower bound results in Theorem \ref{thm:lower-bound} also corroborate the ``parameter counting'' intuition,
and match the upper bound results in Theorem \ref{thm:upper-bound} up to logarithmic factors, under common scenarios and settings.
More specifically, for the $F^{\cavg}$ network, Eq.~(\ref{eq:lower-bound-conv1}) matches Eq.~(\ref{eq:upper-bound-conv1}) up to $O(\sqrt{\log d})$ terms;
for the $F^{\cw}$ network, Eq.~(\ref{eq:lower-bound-conv2}) matches Eq.~(\ref{eq:upper-bound-conv2}) up to $O(\sqrt{\log d})$ terms,
when $s=\Omega(m)$ (and therefore $m/s=O(1)$);
for the $F^{\rnn}$ network, Eq.~(\ref{eq:lower-bound-rnn}) matches Eq.~(\ref{eq:upper-bound-rnn}) up to $O(\sqrt{\log(Ld)})$ terms,
provided that $r=O(L)$ and $L=O(d)$.

To prove Theorem \ref{thm:lower-bound}, 
we reduce it via standard results for lower bounding the minimax risk 
to the problem of finding ``free segments'' in certain structured linear models (Lemma \ref{lem:free-vary}).
Such ``free segments'' are then analyzed in a case-by-case manner for the three networks structures $F^{\cavg},F^{\cw},F^{\rnn}$ considered. 




\section{Proofs of upper bounds}\label{sec:proof-upper}

While Theorem \ref{thm:upper-bound} technically consists of three different upper bounds,
their proofs are similar to each other and therefore we decide to state the proofs in a unified framework,
presented in this section.
Some technical proofs are also deferred to the appendix for a cleaner presentation.

The proof can be roughly divided into three parts.
In the first part, we use the standard statistical analysis of least-squares estimators,
which uses the ``basic inequality'' to translate the task of upper bounding prediction error into
upper bounding the covering number of a suitably self-normalized empirical process \citep{geer2000empirical}.
In the second part, we use restricted eigenvalue arguments similar to \citep{bickel2009simultaneous} to further simplify 
the self-normalized parameter class constructed in the first step.
Finally, in the last part which is the most important step, we use a novel ``linear subspace'' argument
to derive a relatively tight upper bound on the covering number of the desired self-normalized parameter class.

\subsection{Structured linear models and the basic inequality}

Our first observation is that all three models ($F^{\cavg}$, $F^{\cw}$ and $F^{\rnn}$)
are \emph{structured linear models},
meaning that they can be written as a linear regression model with additional structures imposed on the linear regressors.
More specifically, we have the following proposition:
\begin{proposition}
Define vectors $z^i$ and $\theta$ as following:
\begin{enumerate}
\item For $F^{\cavg}$, $z^i := X^i$ and $\theta := \sum_{\ell=0}^{r-1}\sS_s^\ell(w)$, where
$$
\sS_s^\ell(w) = [\underbrace{0, \ldots, 0}_{\ell s\;\; \text{zeros}}, w_{1}, \ldots, w_{m}, 0, \ldots, 0] \in \mathbb R^d;
$$
\item For $F^{\cw}$, $z^i := X^i$ and $\theta := \sum_{\ell=0}^{r-1}a_\ell\sS_s^\ell(w)$;
\item For $F^{\rnn}$, $z^i := (x_1^i\; x_2^i\; \ldots \;x_L^i) \in\mathbb R^{Ld}$ and $\theta := (\vct 1^\top A^{L-1}B~~ \;\;\vct 1^\top A^{L-2}B~~\; \;\ldots\;\; ~~\vct 1^\top B)$.
\end{enumerate}
Then it holds for any $F\in\{F^{\cavg}, F^{\cw}, F^{\rnn}\}$ that $F(X^i;\theta) \equiv \langle z^i,\theta\rangle$.
\label{prop:structured-linear-regression}
\end{proposition}

Proposition \ref{prop:structured-linear-regression} is easily verified using definitions and elementary algebra.
For notational simplicity, we will also use $D$ to denote the dimension of the structured linear model induced by certain types of neural networks.
In particular, for $F^{\cavg},F^{\cw}$ we have $D=d$, and for $F^{\rnn}$ we have $D=Ld$.


Let $\hat\theta\in\mathbb R^D$ be the least-squares estimator on training data $\{X^i,Y^i\}_{i=1}^n$.
Define the \emph{empirical norm} of and $D$-dimensional vector $\vartheta$ as
\begin{equation}
\|\vartheta\|_X^2 := \frac{1}{n}\sum_{i=1}^n \big|\langle z^i, \vartheta\rangle\big|^2.
\label{eq:empirical-norm}
\end{equation}
Because $\hat\theta$ minimizes the least-squares objective as defined in Eq.~(\ref{eq:ls}), we have
$$
\frac{1}{n}\sum_{i=1}^n (Y_i - \langle z^i, \hat\theta\rangle)^2 \leq \frac{1}{n}\sum_{i=1}^n (Y_i - \langle z^i, \theta\rangle)^2.
$$
Because $Y_i = \langle z^i,\theta\rangle + \xi_i$, the above inequality is reduced to
$$
\frac{1}{n}\sum_{i=1}^n (\langle z^i, \theta-\hat\theta\rangle+\xi_i)^2 \leq \frac{1}{n}\sum_{i=1}^n \xi_i^2.
$$
Re-arranging terms and canceling the $\sum_{i=1}\xi_i^2/n$ on both sides of the above inequality, we obtain
\begin{equation}
\|\hat\theta - \theta\|_X^2 \leq \frac{2}{n}\sum_{i=1}^n \xi_i\langle z^i, \hat\theta-\theta\rangle.
\label{eq:basic-ineq}
\end{equation}


\subsection{Self-normalized emprical process and the Dudley's integral}

Let $\Theta\subseteq\mathbb R^{D}$ be the parameter set.
That is, a $D$-dimensional vector $\theta$ belongs to $\Theta$ if and only if there exist 
a parameter configuration yielding $\theta$ in the structured linear model defined in Proposition \ref{prop:structured-linear-regression}. 
Define the \emph{self-normalized} parameter set $\overline\Theta_X$ as
\begin{equation}
\overline\Theta_X := \left\{\phi=\theta-\theta': \theta,\theta'\in\Theta, \|\phi\|_X\leq 1\right\}.
\label{eq:self-normalized-X}
\end{equation}
Define $\mathbb G_X(\phi) := \frac{1}{n}\sum_{i=1}^n \xi_i\langle z^i,\phi\rangle$
as the empirical process associated with $\phi$.
We have the following lemma:
\begin{lemma}
For the parameter sets $\Theta$ defined in Proposition~\ref{prop:structured-linear-regression} we have that:
\begin{align*}
\sup_{\widehat{\theta} \in \Theta} \sum_{i=1}^n \xi_j \langle z^i, \hat\theta-\theta\rangle \leq \|\widehat{\theta} - \theta\|_X 
\sup_{\phi \in \Theta_X} \sum_{i=1}^n \xi_j \langle z^i, \phi\rangle.
\end{align*}
\label{lem:star}
\end{lemma}
This lemma essentially follows by arguing that for each possible $\widehat{\theta} - \theta$ for $\widehat{\theta}, \theta \in \Theta$, there is a corresponding vector $\phi \in \Theta_X$ such that,
\begin{align*}
\sum_{i=1}^n \xi_j \langle z^i, \hat\theta-\theta\rangle = \|\widehat{\theta} - \theta\|_X 
 \sum_{i=1}^n \xi_j \langle z^i, \phi\rangle.
\end{align*}
We defer the proof to the appendix.
 As a consequence of this lemma, by canceling out a $\|\hat\theta-\theta\|_X$ term on both sides of Eq.~(\ref{eq:basic-ineq}), we have
\begin{equation}
\|\hat\theta-\theta\|_X \leq 2 \cdot \sup_{\phi\in\overline\Theta_X}\mathbb G_X(\phi).
\label{eq:self-normalized-ineq}
\end{equation}
Finally, note that for any $\phi,\phi'\in\overline\Theta_X$, $\mathbb G_X(\phi)-\mathbb G_X(\phi')$ is a centered sub-Gaussian random variable
with sub-Gaussian parameter upper bounded by $\sigma^2\|\phi-\phi'\|_X^2$. 
Subsequently, using Dudley's entropy integral \citep{dudley1967sizes}, we have
\begin{equation}
\mathbb E \sup_{\phi\in\overline\Theta_X}\mathbb G_X(\phi) \lesssim \frac{\sigma}{\sqrt{n}} \int_0^\infty\sqrt{\log N(\epsilon;\overline\Theta_X,\|\cdot\|_X)}\ud\epsilon,
\label{eq:dudley}
\end{equation}
where $N(\epsilon;\overline\Theta_X,\|\cdot\|_X)$ is the covering number of $\overline\Theta_X$ in $\|\cdot\|_X$ (i.e.,
the size of the smallest set $\mathcal H$ such that $\sup_{\phi\in\overline\Theta_X}\inf_{\phi'\in\mathcal H}\|\phi-\phi'\|_X\leq \epsilon$).

Combining Eqs.~(\ref{eq:self-normalized-ineq}) and (\ref{eq:dudley}) we arrive at the following main inequality of this part of the proof:
\begin{equation}
\mathbb E\left[\|\hat\theta-\theta\|_X\right] \lesssim \frac{\sigma}{\sqrt{n}}\int_0^\infty\sqrt{\log N(\epsilon;\overline\Theta_X,\|\cdot\|_X)}\ud\epsilon.
\label{eq:basic-ineq-X}
\end{equation}

\subsection{Restricted eigenvalues}

The $\|\phi\|_X\leq 1$ constraint in the definition of $\overline\Theta_X$ is quite difficult to exploit,
and we hope to replace it with simpler constraints such as $\|\phi\|_2\leq 1$.
Traditionally, this is done by bounding the eigenvalues of the \emph{sample covariance} of $\{X^i\}_{i=1}^n$
and their corresponding expanded form $\{z^i\}_{i=1}^n$.
Unfortunately, in the regime of $n\ll D$ the sample covariance of $\{z^i\}_{i=1}^n$
is certainly rank-deficient, making such an argument void.

To overcome this difficulty, we introduce  \emph{restricted eigenvalues} which are used extensively in high-dimensional statistics \citep{bickel2009simultaneous,wainwright2009sharp}.
\begin{definition}[Restricted Eigenvalues]
For a data set $\{z^i\}_{i=1}^n\subseteq\mathbb R^{D}$, its smallest and largest restricted eigenvalues with respect to 
a parameter class $\Phi\subseteq\mathbb R^{D}$ is defined as
\begin{align}
\lambda_{\min}(\{z^i\}_{i=1}^n;\Phi) &:= \inf_{\phi\in\Phi}\|\phi\|_X^2/\|\phi\|_2^2;\label{eq:re-min}\\
\lambda_{\max}(\{z^i\}_{i=1}^n;\Phi) &:= \sup_{\phi\in\Phi}\|\phi\|_X^2/\|\phi\|_2^2.\label{eq:re-max}
\end{align}
\end{definition}
For any $\rho>0$, define $\overline\Theta_2(\rho)\subseteq\mathbb R^{D}$ as
\begin{equation}
\overline\Theta_2(\rho) := \left\{\phi = \theta-\theta': \theta,\theta'\in\Theta, \|\phi\|_2\leq \rho\right\}.
\label{eq:self-normalized-2}
\end{equation}
Comparing the definitions of $\overline\Theta_2(\rho)$ with $\overline\Theta_X$, the major difference is in the normalizing norm:
in the definition of $\overline\Theta_2(\rho)$ the $\|\cdot\|_2$ norm is used to constrain the parameter set while in $\overline\Theta_X$
the empirical norm $\|\cdot\|_X$ is used.
Also, the definition of $\overline\Theta_2(\rho)$ involves an additional ``radius'' parameter $\rho>0$,
allowing for more flexibility in later proofs.

The following lemma establishes restricted eigenvalues of $\{z^i\}_{i=1}^n$ with respect to $\overline\Theta_2(\rho)$, 
provided that the training set size $n$ is sufficiently large.
\begin{lemma}
Suppose $\{z^i\}_{i=1}^n$ are sub-Gaussian random vectors with variance parameter $Z^2$.
For any $\rho>0$, $\delta\in(0,1/2]$ and $\epsilon\in(0,1/2]$, with probability $1-\delta$ it holds that
\begin{align*}
\lambda_{\min}(\{z^i\}_{i=1}^n; \overline\Theta_2(\rho))& \geq \frac{c}{4} - O(Z\sqrt{\log(n/\delta)})\cdot\left(\epsilon + \sqrt{\frac{\log N(\epsilon;\overline\Theta_2(1),\|\cdot\|_2)\log(1/\delta)}{n}}\right) ;\\
\lambda_{\max}(\{z^i\}_{i=1}^n; \overline\Theta_2(\rho)) &\leq 4C + O(Z\sqrt{\log(n/\delta)})\cdot\left(\epsilon + \sqrt{\frac{\log N(\epsilon;\overline\Theta_2(1),\|\cdot\|_2)\log(1/\delta)}{n}}\right).
\end{align*}
\label{lem:re}
\end{lemma}
\begin{remark}
For $\{z^i\}_{i=1}^n$ defined in Proposition \ref{prop:structured-linear-regression}, their sub-Gaussian parameters $Z^2$ can be bounded as 
$Z^2\leq C^2$ for all $F^{\cavg},F^{\cw}$ and $F^{\rnn}$,
where $C$ is the constant in Assumption (A2).
\end{remark}
The proof of Lemma \ref{lem:re} is quite involved and we defer it to the appendix.
Note that the right sides of both inequalities in Lemma \ref{lem:re} do \emph{not} depend on $\rho$.
This is natural and is expected, because the restricted eigenvalues defined in Eqs.~(\ref{eq:re-min},\ref{eq:re-max}) are scale-invariant.

\subsection{Covering number upper bounds}

The objective of this section is to give upper bounds on covering numbers of self-normalized parameter classes.
Since the parameter classes depend heavily on the underlying network structures, we derive their corresponding covering numbers separately.
However, the derivation of all covering numbers will rely on a crucial lemma bounding the covering number of \emph{low-dimensoinal linear subspaces},
which we state below:
\begin{lemma}
Fix $q$, $k\leq q$, $\rho>0$, and $\epsilon'\in(0,1/2]$.
There exists a set $\mathcal W$ consisting of a finite number of $k$-dimensional linear subspaces in $\mathbb R^q$
that satisfies the following:
for any $K$-dimensional linear subspace $S$ in $\mathbb R^q$, there exists $S'\in\mathcal W$ such that
\begin{equation}
\sup_{u\in S, \|u\|_2\leq \rho} \inf_{v\in S',\|v\|_2\leq\rho} \|u-v\|_2 \leq \epsilon'.
\label{eq:linear-subspace-cover}
\end{equation}
Furthermore, the size of $\mathcal W$ can be upper bounded as $\log|\mathcal W|\lesssim kq\log(\rho q/\epsilon')$.
\label{lem:linear-subspace-cover}
\end{lemma}
The proof of Lemma \ref{lem:linear-subspace-cover} is deferred to the appendix.

\subsubsection{Covering number for $F^{\cavg}$}
\begin{lemma}
For $\Theta$ induced by $F^{\cavg}$ and any $\rho>0$, $\epsilon\in(0,1]$, it holds thaat
$$\log(\epsilon;\overline\Theta_2(\rho),\|\cdot\|_2) \lesssim m\log(\rho d/\epsilon).
$$
\label{lem:covering-cavg}
\end{lemma}
\begin{proof}
Let $\theta,\theta'\in\Theta$ be $d$-dimensional parameterizations of $w$ and $w'$, respectively,
as derived in Proposition \ref{prop:structured-linear-regression}.
Denote also $\theta(\mathcal I_j)$ for $j\in\{1,\ldots,d/s\}$ as the $j$th $s$-dimensional segment of $\theta\in\mathbb R^d$, corresponding to the segment starting with the $((j-1)s+1)$-th
entry and ending with the $js$-th entry.
Denote also $w(\mathcal I_j)$ for $j\in\{1,\ldots,J\}$ as the $j$th $s$-dimensional segment of $w\in\mathbb R^m$, where $J=m/s$.
For $\phi=\theta-\theta'$,
it is easy to verify that
\begin{equation}
\phi(\mathcal I_j) =\sum_{\ell=1}^{\min\{j,J\}} w(\mathcal I_\ell)-w'(\mathcal I_\ell),\;\;\;\;\;\; j=1,2,\ldots,d/s.
\label{eq:phiseg-cavg-1}
\end{equation}

Because $\|\phi\|_2\leq\rho$, we have that $\|\phi(\mathcal I_j)\|_2\leq \rho$ for all $j=1,\ldots,d/s$, and subsequently
\begin{equation}
\|w(\mathcal I_j)-w'(\mathcal I_j)\|_2 \leq \sum_{\ell=1}^j \|\phi(\mathcal I_j)\|_2 \leq J\rho, \;\;\;\;\;\;j=1,2,\ldots,J.
\label{eq:phiseg-cavg-2}
\end{equation}
Therefore,
\begin{equation}
\|w-w'\|_2 \leq \sum_{j=1}^J \|w(\mathcal I_j)-w'(\mathcal I_j)\|_2 \leq J^2\rho.
\label{eq:phiseg-cavg-3}
\end{equation}

Next construct a covering set $\mathcal H$ such that for any $x\in\mathbb R^m$, $\|x\|_2\leq J^2\rho$,
$\min_{z\in\mathcal H}\|x-z\|_2\leq\epsilon'$ for some parameter $\epsilon'>0$ to be specified later.
Such construction is standard (see, e.g., \cite{geer2000empirical}), and the size of $\mathcal H$ can be upper bounded by 
$\log|\mathcal H|\lesssim m\log(J\rho/\epsilon')$.
Because $w-w'\in\mathbb R^m$ satisfies $\|w-w'\|_2\leq J^2\rho$, it holds that 
\begin{equation}
\min_{v\in\mathcal H} \|(w-w')-v\|_2 \leq \epsilon'.
\label{eq:phiseg-cavg-4}
\end{equation}

Define $\tilde v\in\mathbb R^d$ as $\tilde v(\mathcal I_j)=\sum_{\ell=1}^{\min\{j,J\}}v(\mathcal I_\ell)$ for $j=1,\ldots,d/s$.
Eq.~(\ref{eq:phiseg-cavg-4}) and (\ref{eq:phiseg-cavg-1}) imply that
$$
\|\phi-\tilde v\|_2 \leq \sum_{j=1}^{d/s}\max_{\ell\leq J}\|(w(\mathcal I_\ell)-w'(\mathcal I_\ell))-v(\mathcal I_\ell)\|_2\leq d\epsilon'/s \leq d\epsilon'.
$$
Setting $\epsilon'=\epsilon/d$ we have $\log|\mathcal H|\lesssim m\log(d\rho/\epsilon')$, which completes the proof of Lemma \ref{lem:covering-cavg}.
\end{proof}

\subsubsection{Covering number for $F^{\cw}$}

\begin{lemma}
For $\Theta$ induced by $F^{\cw}$ and any $\rho>0$, $\epsilon\in(0,1]$, it holds thaat
$$\log(\epsilon;\overline\Theta_2(\rho),\|\cdot\|_2) \lesssim \min\{d, m+(d/s)\times (m/s)\}\cdot \log(\rho d/\epsilon).
$$
\label{lem:covering-cw}
\end{lemma}
\begin{proof}
Let $\theta,\theta'\in\Theta$ be $d$-dimensional parameterizations of $w,a$ and $w',a'$, respectively,
as derived in Proposition \ref{prop:structured-linear-regression}.
Denote also $\theta(\mathcal I_j)$ for $j\in\{1,\ldots,d/s\}$ as the $j$th $s$-dimensional segment of $\theta\in\mathbb R^d$, corresponding to the segment starting with the $((j-1)s+1)$-th
entry and ending with the $js$-th entry.
Denote also $w(\mathcal I_j)$ for $j\in\{1,\ldots,J\}$ as the $j$th $s$-dimensional segment of $w\in\mathbb R^m$, where $J=m/s$.
For $\phi=\theta-\theta'$,
it is easy to verify that
\begin{equation}
\phi(\mathcal I_j) =\sum_{\ell=1}^{\min\{j,J\}} a_\ell w(\mathcal I_\ell)-a_\ell'w'(\mathcal I_\ell),\;\;\;\;\;\; j=1,2,\ldots,d/s.
\label{eq:phiseg-cw-1}
\end{equation}

Clearly, Eq.~(\ref{eq:phiseg-cw-1}) implies that
\begin{equation}
\phi(\mathcal I_j) \in \SPAN\{w(\mathcal I_1),w'(\mathcal I_1),\ldots,w(\mathcal I_J),w'(\mathcal I_J)\}, \;\;\;\;\;\; j=1,2,\ldots,d/s,
\label{eq:phiseg-cw-2}
\end{equation}
where $J=m/s$.
This observation motivates a two-step construction of covering sets of $\overline\Theta_2(\rho)$:
by first constructing a covering set of all $2J$-dimensional linear subspaces in $\mathbb R^s$, and then 
covering all vectors within each linear subspace whose $\ell_2$ norms are upper bounded by $\rho$.

Choosing $q=s$ and $k=\min\{2J,s\}$ in Lemma \ref{lem:linear-subspace-cover}, we have a covering set $\mathcal W$ of $2J$-dimensional linear subspaces in $\mathbb R^s$
with size upper bounded by $\log|\mathcal W|\lesssim Js\log(\rho s/\epsilon')=m\log(\rho s/\epsilon')$.
Next, for each linear subspace $W\in\mathcal W$, construct a finite covering set $\mathcal H(W)\subseteq W$ such that 
$\sup_{x\in W,\|x\|_2\leq\rho}\min_{v\in\mathcal H(W)} \|x-v\|_2 \leq \epsilon'$.
Because $W\subseteq\mathbb R^s$ and $\dim(W)\leq k\leq 2J$, such a finite covering set $\mathcal H(W)$ exists with $\log|\mathcal H(W)|\lesssim J\log(\rho/\epsilon') = (m/s)\times \log(\rho/\epsilon')$.

Next, construct covering set $\mathcal M\subseteq\mathbb R^{d}$ as
$$
\mathcal M = \bigcup_{W\in\mathcal W} \underbrace{\mathcal H(W)\times\ldots\times\mathcal H(W)}_{\text{$d/s$ times}}.
$$
By Eq.~(\ref{eq:phiseg-cw-2}) and the covering properties of $\mathcal W$ and $\mathcal H(W)$, it holds that
$$
\sup_{\theta,\theta'\in\Theta} \min_{v\in\mathcal M} \|(\theta-\theta')-v\|_2 \leq 2d\epsilon'/s \leq 2d\epsilon'.
$$
Furthermore, the size of $\mathcal M$ can be upper bounded by $\log|\mathcal M| \leq \log|\mathcal W| + \max_{W\in\mathcal W}L\log|\mathcal H(W)|
\lesssim m\log(\rho s/\epsilon') + (d/s)\times (m/s)\times \log(\rho/\epsilon')$.
Setting $\epsilon' = \epsilon/(2d)$, we obtain a covering $\mathcal M$ of $\overline\Theta_2(\rho)$ with respect to $\|\cdot\|_2$
up to precision $\epsilon$, with size $\log|\mathcal M|\lesssim (m+(d/s)\times (m/s))\log(\rho d/\epsilon)$.

Finally, note that $\log N(\epsilon;\overline\Theta_2(\rho),\|\cdot\|_2) \lesssim d\log(\rho/\epsilon)$ always holds
because $\overline\Theta_2(\rho)\subseteq \{x\in\mathbb R^d: \|x\|_2\leq \rho\}$.
This completes the proof of Lemma \ref{lem:covering-cw}.
\end{proof}

\subsubsection{Covering number for $F^{\rnn}$}
\begin{lemma}
For $\Theta$ induced by $F^{\rnn}$ and any $\rho>0$, $\epsilon\in(0,1]$, it holds that
$$\log(\epsilon;\overline\Theta_2(\rho),\|\cdot\|_2) \lesssim (d+L)\cdot \min\{d,r\}\log(\rho Ld/\epsilon).
$$
\label{lem:covering-rnn}
\end{lemma}
\begin{proof}
Let $\theta,\theta'\in\Theta$ be $Ld$-dimensional parameterizations of $A,B$ and $A',B'$, respectively,
as derived in Proposition \ref{prop:structured-linear-regression}.
Denote also $\theta(\mathcal I_j)$ for $j\in\{1,\ldots,m\}$ as the $j$th $d$-dimensional segment of $\theta$, corresponding to the segment starting with the $((j-1)d+1)$-th
entry and ending with the $jd$-th entry.
By definition, $\phi=\theta-\theta'$ satisfies 
\begin{equation}
\phi(\mathcal I_j) = \theta(\mathcal I_j)-\theta'(\mathcal I_j) = \vct 1^\top A^{L-j} B - \vct 1^\top (A')^{L-j} B'.
\label{eq:phiseg-1}
\end{equation}

Let $b_1,\ldots,b_r,b_1',\ldots,b_r'\in\mathbb R^d$ denote the rows of $B$ and $B'$, respectively.
Eq.~(\ref{eq:phiseg-1}) then implies
\begin{equation}
\phi(\mathcal I_j) \in \SPAN\{b_1,\ldots,b_r,b_1',\ldots,b_r'\}, \;\;\;\;\;\; j=1, 2, \ldots, m.
\label{eq:phiseg-2}
\end{equation}

Eq.~(\ref{eq:phiseg-2}) motivates a two-step construction of covering sets of $\overline\Theta_2(\rho)$: 
by first constructing a covering set of all $2r$-dimensional linear subspaces in $\mathbb R^d$, and then 
covering all vectors within each linear subspace whose $\ell_2$ norms are upper bounded by $\rho$.



Choosing $q=d$ and $k=\min\{2r,d\}$ in Lemma \ref{lem:linear-subspace-cover}, we have a covering set $\mathcal W$ of $2r$-dimensional linear subspaces in $\mathbb R^d$
with size upper bounded by $\log|\mathcal W|\lesssim \min\{d,r\}\cdot d\log(\rho d/\epsilon')$.
Next, for each linear subspace $W\in\mathcal W$, construct a finite covering set $\mathcal H(W)\subseteq W$ such that 
$\sup_{x\in W,\|x\|_2\leq\rho}\min_{v\in\mathcal H(W)} \|x-v\|_2 \leq \epsilon'$.
Because $W\subseteq\mathbb R^d$ and $\dim(W)\leq k=\min\{2r,d\}$, such a finite covering set $\mathcal H(W)$ exists with $\log|\mathcal H(W)|\lesssim \min\{2r,d\}\log(\rho/\epsilon')$.

Next, construct covering set $\mathcal M\subseteq\mathbb R^{Ld}$ as
$$
\mathcal M = \bigcup_{W\in\mathcal W} \underbrace{\mathcal H(W)\times\ldots\times\mathcal H(W)}_{\text{$L$ times}}.
$$
By Eq.~(\ref{eq:phiseg-2}) and the covering properties of $\mathcal W$ and $\mathcal H(W)$, it holds that
$$
\sup_{\theta,\theta'\in\Theta} \min_{v\in\mathcal M} \|(\theta-\theta')-v\|_2 \leq 2L\epsilon'.
$$
Furthermore, the size of $\mathcal M$ can be upper bounded by $\log|\mathcal M| \leq \log|\mathcal W| + \max_{W\in\mathcal W}L\log|\mathcal H(W)|
\lesssim \min\{d,r\}(d\log(\rho d/\epsilon') + L\log(\rho/\epsilon'))$.
Setting $\epsilon' = \epsilon/(2L)$, we obtain a covering $\mathcal M$ of $\overline\Theta_2(\rho)$ with respect to $\|\cdot\|_2$
up to precision $\epsilon$, with size $\log|\mathcal M|\lesssim (d+L)\cdot \min\{d,r\}\log(\rho Ld/\epsilon)$.
\end{proof}

\subsection{Putting everything together}

In this section we complete the proofs of the three minimax upper bounds in Theorem \ref{thm:upper-bound}.
%
%
First we derive conditions under which $\lambda_{\min}(\{z_i\}_{i=1}^n;\overline\Theta_2(\rho))\geq c/8$
and $\lambda_{\max}(\{z_i\}_{i=1}^n;\overline\Theta_2(\rho))\leq 8C$ with high probability.
Select $\epsilon=\kappa c/C\sqrt{\log(n/\delta)}$ for some sufficiently small constant $\kappa>0$, so that 
the $\epsilon\cdot O(C^2\sqrt{\log n/\delta})$ term in Lemma \ref{lem:re} is upper bounded by $c/16$.
Using the upper bounds on $\log N(\epsilon;\overline\Theta_2(\rho),\|\cdot\|_2)$ in Lemmas \ref{lem:covering-cavg}, \ref{lem:covering-cw} and \ref{lem:covering-rnn}, 
it is easy to verify that, if $n$ satisfies
\begin{eqnarray*}
\text{For $F^{\cavg}$}: & & n \gtrsim c^{-1}C^2m\cdot \log(c^{-1}C d\log(n/\delta))\log^2(n/\delta),\\
\text{For $F^{\cw}$}: && n \gtrsim  c^{-1}C^2 \min\{d,m+(m/s)\times (d/s)\}\cdot \log(c^{-1}Cd\log(n/\delta))\log^2(n/\delta),\\
\text{For $F^{\rnn}$}: & & n \gtrsim c^{-1}C^2(d+L)\min\{d,r\}\cdot \log(c^{-1}C Ld\log(n/\delta))\log^2(n/\delta),
\end{eqnarray*}
then with probability $1-\delta$, both $\lambda_{\min}(\{z_i\}_{i=1}^n;\overline\Theta_2(\rho))\geq c/8$ and 
$\lambda_{\max}(\{z_i\}_{i=1}^n;\overline\Theta_2(\rho))\leq 8C$ hold.
The rest of the proof will be conditioned on the success event that these two RE-type inequalities hold.

When the RE conditions hold, we have $(c/16)\|\phi\|_2^2\leq \|\phi\|_X^2\leq 16C\|\phi\|_2^2$ for all $\phi\in\Theta$.
The covering number $\log N(\epsilon;\overline\Theta_X,\|\cdot\|_X)$ can then be upper bounded as
$$
\log N(\epsilon;\overline\Theta_X,\|\cdot\|_X) \leq \log N(\epsilon/4\sqrt{C};\overline\Theta_2(4/\sqrt{c}),\|\cdot\|_2).
$$

Invoking Lemmas \ref{lem:covering-cavg}, \ref{lem:covering-cw} and \ref{lem:covering-rnn}, we have
\begin{eqnarray*}
\text{For $F^{\cavg}$}: && \log N(\epsilon;\overline\Theta_X,\|\cdot\|_X) \lesssim m\log(c^{-1}C d/\epsilon),\\
\text{For $F^{\cw}$}:&& \log N(\epsilon;\overline\Theta_X,\|\cdot\|_X) \lesssim \min\{d, m+(m/s)\times (d/s)\}\cdot \log(c^{-1}Cd/\epsilon),\\
\text{For $F^{\rnn}$}: && \log N(\epsilon;\overline\Theta_X,\|\cdot\|_X) \lesssim(d+L)\cdot \min\{d,r\}\log(c^{-1}C Ld/\epsilon).
\end{eqnarray*}

Incorporating the above inequalities into Eq.~(\ref{eq:basic-ineq-X}),
and noting that $\int_0^{\infty}\sqrt{\max\{\log(1/\epsilon),0\}}\ud\epsilon=O(1)$, we complete the proof of Theorem \ref{thm:upper-bound}.

\section{Proofs of lower bounds}

To prove the minimax lower bounds in Theorem \ref{thm:lower-bound} we use the following result from \cite{tsybakov2009introduction}:
\begin{lemma}[\cite{tsybakov2009introduction}]
Let $\Theta_M=(\theta_0,\theta_1,\ldots,\theta_M)$ be a finite collection of parameters and let $P_j$ be the distribution induced by parameter $\theta_j$, for $j\in\{0,\ldots,M\}$.
Let also $\sd:\Theta_M\times\Theta_M\to\mathbb R^+$ be a semi-distance. 
Suppose the following conditions hold:
\begin{enumerate}
\item $\sd(\theta_j,\theta_k)\geq 2\rho>0$ for all $j,k\in\{0,\ldots,M\}$;
\item $P_j\ll P_0$ for every $j\in\{1,\ldots,M\}$; \footnote{$P\ll Q$ means that the support of $P$ is contained in the support of $Q$.}
\item $\frac{1}{M}\sum_{j=1}^M\kl(P_j\|P_0) \leq \gamma\log M$;
\end{enumerate}
then the following bound holds:
\begin{equation}
\inf_{\hat\theta}\sup_{\theta_j\in\Theta_M} \Pr_j\left[\sd(\hat \theta,\theta_j)\geq \rho\right] \geq \frac{\sqrt{M}}{1+\sqrt{M}}\left(1-2\gamma-2\sqrt{\frac{\gamma}{\log M}}\right).
\end{equation}
\label{lem:fano}
\end{lemma}

With Lemma \ref{lem:fano}, the problem of lower bounding the minimax risk $\mathfrak M(n;\cdot)$ can be 
reduced to the question of constructing appropriate ``adversarial'' parameter sets $\Theta'\in\Theta$, with upper bounded KL divergence and 
lower bounded distance measure $\sd(\cdot,\cdot)$ between considered parameters.
Because in our lower bounds the data points $\{x^i\}$ (for $F^{\cavg},F^{\cw}$) or $\{x_t^i\}$ (for $F^{\rnn}$)
follow isotropic Gaussian distributions, and the noise variables are distributed as $\mathcal N(0,\sigma^2)$, we have the following corollary
as a consequence of Lemma \ref{lem:fano}:
\begin{corollary}
Let $\Theta\subseteq\mathbb R^D$ be the parameter set induced by network $F$, as derived in Proposition \ref{prop:structured-linear-regression}.
For any finite subset $\Theta'=\{\theta_0,\theta_1,\ldots,\theta_M\}\subseteq\Theta$, 
denote $\rho_{\min} := \min_{j>0}\|\theta_0-\theta_j\|_2/2$ and $\rho_{\avg}^2 := \frac{1}{M}\sum_{i=1}^M\|\theta_i-\theta_0\|_2^2$.
Then for any $n$,
$$
\inf_{\hat\theta_n}\sup_{\theta\in\Theta} \mathbb E_\mu[\|\hat\theta_n-\theta\|_2] \geq \rho_{\min}\times \frac{\sqrt{M}}{1+\sqrt{M}}\left(1-\frac{n\rho_{\avg}^2}{\sigma^2\log M} - 2\sqrt{\frac{n\rho_{\avg}^2}{2\sigma^2\log^2 M}}\right).
$$
\label{cor:fano}
\end{corollary}

The proof of Corollary \ref{cor:fano} involves some routine verifications of the conditions in Lemma \ref{lem:fano},
and is placed in the appendix.

The following lemma considers the special case when a certain number of components in $\Theta\subseteq\mathbb R^D$ are allowed to vary freely.
\begin{lemma}
Let $\Theta\subseteq\mathbb R^D$ be the parameter set induced by network $F$, and $\mathcal I\subseteq[D]$ be a subset of components.
Suppose for any $u\in\mathbb R^{|\mathcal I|}$, there exists $\theta\in\Theta$ such that $\theta$ restricted to $\mathcal I$ equals $u$.
Then there exists a finite subset $\Theta'\subseteq\Theta$ as in Corollary \ref{cor:fano},
with $\log M\asymp |\mathcal I|$ and $\rho_{\min}\asymp \rho_{\avg} \asymp \sqrt{|\mathcal I|}\epsilon$
for any $\epsilon>0$.
\label{lem:free-vary}
\end{lemma}

Lemma \ref{lem:free-vary} will be proved in the appendix, based on the standard construction of separable constant-weight codes (e.g., \citep[Lemma 9]{wang2016noise}, \citep[Theorem 7]{graham1980lower}).
It will play a central role in the proofs of lower bounds for the networks $F^{\cavg}$, $F^{\cw}$ and $F^{\rnn}$, as we state separately below.


\subsection{Proof of minimax lower bound for network $F^{\cavg}$}

We shall prove the following lemma, showing that the first $m$ components of $\theta\in\mathbb R^d$ can vary freely under $F^{\cavg}$.
\begin{lemma}
Let $\mathcal I=\{1,\ldots,m\}$ and $\Theta=\{\theta(w):w\in\mathbb R^m\}\subseteq\mathbb R^d$, where $\theta(w) = \sum_{\ell=0}^{(d-m)/s}\sS_s^\ell(w)$
as defined in Proposition \ref{prop:structured-linear-regression}.
Then for any $u\in\mathbb R^m$, there exists $\theta\in\Theta$ such that $\theta$ restricted on $\mathcal I$ equals $u$.
\label{lem:linear-fcavg}
\end{lemma}
\begin{proof}
Recall that $J=m/s$ is a positive integer. Let $u(\mathcal J_1),\ldots,u(\mathcal J_J)$ denote the $J$ segments of $u$, each of length $s$.
To construct the weight vector $w\in\mathbb R^m$, we decompose $w$ into $w(\mathcal J_1),\ldots,w(\mathcal J_J)$ as well, and construct
$$
w(\mathcal J_j) = u(\mathcal J_j) - \sum_{k=0}^{j-1}u(\mathcal J_k), \;\;\;\;\;\; j=0,1,\ldots,J-1.
$$

Because $\theta(w) = \sum_{\ell=0}^{(d-m)/s}\sS_s^\ell(w)$, it is easy to verify that $w\in\mathbb R^m$ constructed above and its corresponding $\theta(w)$
has the same first $m$ components as $u$. The lemma is thus proved.
\end{proof}

With Lemma \ref{lem:linear-fcavg}, Eq.~(\ref{eq:lower-bound-conv1}) in Theorem \ref{thm:lower-bound} immediately follows from Corollary \ref{cor:fano}
and Lemma \ref{lem:free-vary}, with $\epsilon$ in Lemma \ref{lem:free-vary} set as $\epsilon\asymp \sigma\sqrt{m/n}$.

\subsection{Proof of minimax lower bound for network $F^{\cavg}$}

Because $d/s+m\leq 2\max(d/s,m)$, it suffices to prove minimax lower bounds of $\sqrt{\sigma^2m/n}$ and $\sqrt{\sigma^2 d/(sn)}$ separately.
\begin{lemma}
Let $\mathcal I_1=\{1,\ldots,m\}$ and $\mathcal  I_2 = \{1,1+s,\ldots,1+(d/s-1)s\}$.
Let also $\Theta=\{\theta(w,a):w\in\mathbb R^m,a\in\mathbb R^{d/s}\}\subseteq\mathbb R^d$ be the induced parameter space,
where $\theta(w,a) = \sum_{\ell=0}^{(d-m)/s}a_\ell \sS_s^\ell(w)$.
Then for any $\mathcal I\in\{\mathcal I_1,\mathcal I_2\}$ and $u\in\mathbb R^{|\mathcal I|}$,
there exists $\theta\in\Theta$ such that $\theta$ restricted on $\mathcal I$ equals $u$.
\label{lem:linear-fcw}
\end{lemma}
\begin{proof}
We first prove the lemma for $\mathcal I=\mathcal I_1$ and $u\in\mathbb R^{|\mathcal I_1|}$.
Consider $a = (1,0,\ldots,0)$ and $w=u$.
Then $\theta(w,a)=(w, 0, \ldots, 0)$ and therefore the first $m$ components of $w$ equal $u$.

We next prove the lemma for $\mathcal I=\mathcal I_2$ and $u\in\mathbb R^{|\mathcal I_2|}$.
Consider $w=(1,0,\ldots,0)$ and $a=u$..
Then $\theta(w,a) = (a_0,0,\ldots,0,a_1,0,\ldots)$ and therefore $\theta(w,a)$ restricted to $\mathcal I_2$
equal $u$.
\end{proof}

With Lemma \ref{lem:linear-fcavg}, Eq.~(\ref{eq:lower-bound-conv2}) in Theorem \ref{thm:lower-bound} immediately follows from Corollary \ref{cor:fano}
and Lemma \ref{lem:free-vary}, with $\epsilon$ in Lemma \ref{lem:free-vary} set as $\epsilon\asymp \sigma\sqrt{\max\{m,d/s\}/n}$.

\subsection{Proof of minimax lower bound for network $F^{\rnn}$}

We establish the following lemma showing that for the network $F^{\rnn}$ and its equivalent linear parameter $\theta$ defined in Proposition \ref{prop:structured-linear-regression},
the last $\min\{rd,Ld\}$ components of $\theta$ are free to vary.
\begin{lemma}
Let $\mathcal I=\{1,2,\ldots,\min(rd,Ld)\}$ and $\Theta\subseteq\mathbb R^{Ld}$ be the parameter space induced by $F^{\rnn}$, as shown in Proposition \ref{prop:structured-linear-regression}.
Then for any $u\in\mathbb R^{|\mathcal I|}$, there exists $\theta\in\Theta$ such that $\theta$ restricted on $\mathcal I$ equals $u$.
\label{lem:linear-frnn}
\end{lemma}
\begin{proof}
Denote $r'=\min\{r,L\}$ and for any $u\in\mathbb R^{|\mathcal I|}$, let $u(\mathcal J_1),\ldots,u(\mathcal J_{r'})\in\mathbb R^d$
be its $r'$ segments, each of length $d$.
Let also $\theta(\mathcal J_1),\ldots,\theta(\mathcal J_{r'})$ be the corresponding $d$-dimensional segments of the last $r'd$ components of $\theta$,
corresponding to an RNN network $F^{\rnn}$ with weight matrices $A\in\mathbb R^{r\times r}$ and $B\in\mathbb R^{r\times d}$.
By definition, $\theta(\mathcal J_{\ell}) = \vct 1^\top A^{r'-\ell} B$.

Consider diagonal matrix $A=\diag(\{a_i\}_{i=1}^r)$.
Because $\theta(\mathcal J_\ell)=u(\mathcal J_\ell)$ for all $\ell\in[r']$,
we have that
\begin{equation}
\sum_{i=1}^r a_i^\ell b_{ij} = [u(\mathcal J_{r'-\ell+1})]_j, \;\;\;\;\;\;\ell\in\{0,\ldots,r'-1\}, \;\; j\in\{1,\ldots,d\}.
\label{eq:rnn-lowerbound-eq1}
\end{equation} 
Define matrix $G\in\mathbb R^{r\times r'}$ as $G_{i\ell} = \{a_i^\ell\}_{i,\ell=1}^{r,r'}$.
Subsequently, Eq.~(\ref{eq:rnn-lowerbound-eq1}) can be compactly rewritten as
\begin{equation}
G c_j = v_j, \;\;\;\;\;\; j\in\{1,\ldots,d\},
\label{eq:rnn-lowerbound-eq2}
\end{equation}
where $c_j=\{b_{ij}\}_{i=1}^r$ and $v_j = \{[u(\mathcal J_{r'-\ell+1})]_{\ell=0}^{r'-1}\}$ are both $r$-dimensional vectors.
Because $r'\leq r$ and the vectors $\{c_j\}$ ``partition'' the parameter matrix $B\in\mathbb R^{r\times d}$,
to prove the existence of such $\{c_j\}$ for any $\{v_j\}$
we only need to show that the rows of $G$ are linearly independent.
By taking $a_i := i/r$, it is clear that $G$ has linearly independent rows because it is a Vandermonde matrix with distinct roots $\{a_i\}_{i=1}^r$.
\end{proof}

With Lemma \ref{lem:linear-frnn}, Eq.~(\ref{eq:lower-bound-rnn}) in Theorem \ref{thm:lower-bound} follows from Corollary \ref{cor:fano}
and Lemma \ref{lem:free-vary}, with $\epsilon$ in Lemma \ref{lem:free-vary} set as $\epsilon\asymp \sigma\sqrt{\min\{rd,Ld\}/n}$.

\section{Experiments}
\label{sec:exp}

In this section we use simulations to verify our theoretical findings..
We first consider CNNs.
We let the ambient dimension $d$ be $64$ and the input distribution be Gaussian with mean $0$ and identity covariance.
In all plots, CNN represents using convolutional parameterization corresponding to Eq.~\eqref{eqn:conv_filter_avg_pooling} or Eq.~\eqref{eqn:two_layer} and FNN represents using fully connected parametrization.

\begin{figure*}[t!]
	\centering
	\begin{subfigure}[t]{0.29\textwidth}
		\includegraphics[width=\textwidth]{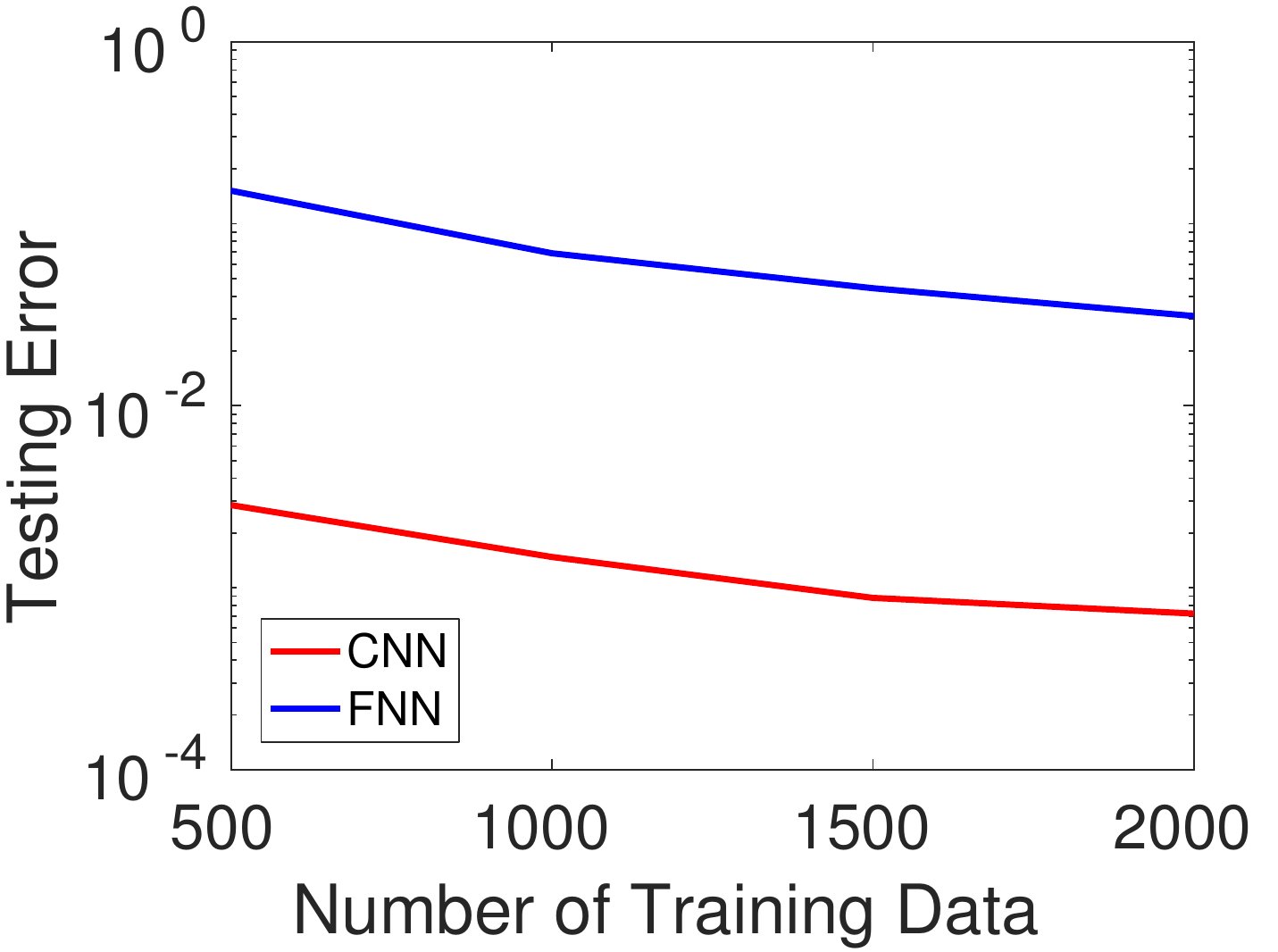}
		\caption{Filter size $m=2$.}
	\end{subfigure}	
	\quad
	\begin{subfigure}[t]{0.29\textwidth}
		\includegraphics[width=\textwidth]{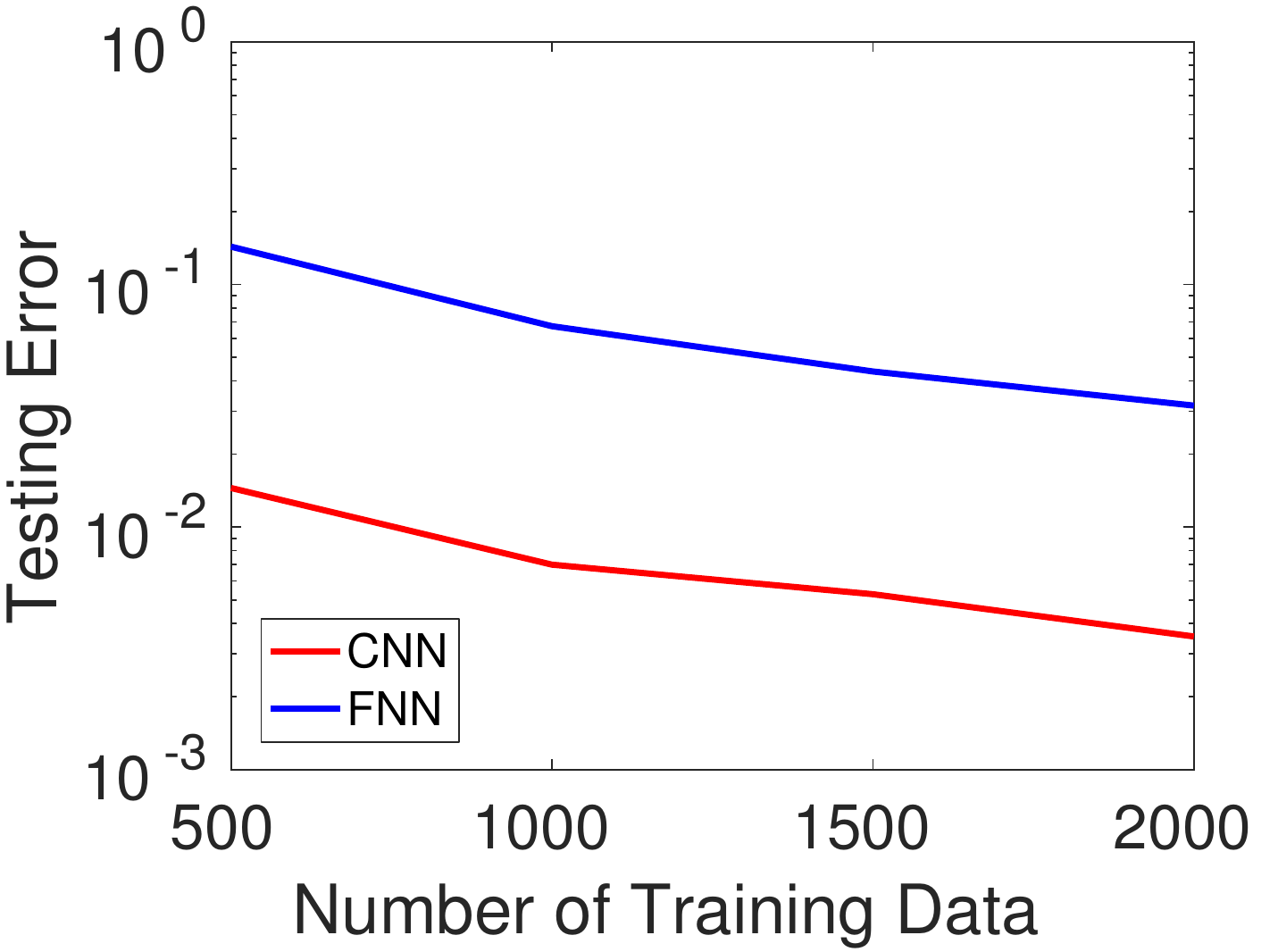}
		\caption{Filter size $m=8$.}
	\end{subfigure}	
	\quad
	\begin{subfigure}[t]{0.29\textwidth}
		\includegraphics[width=\textwidth]{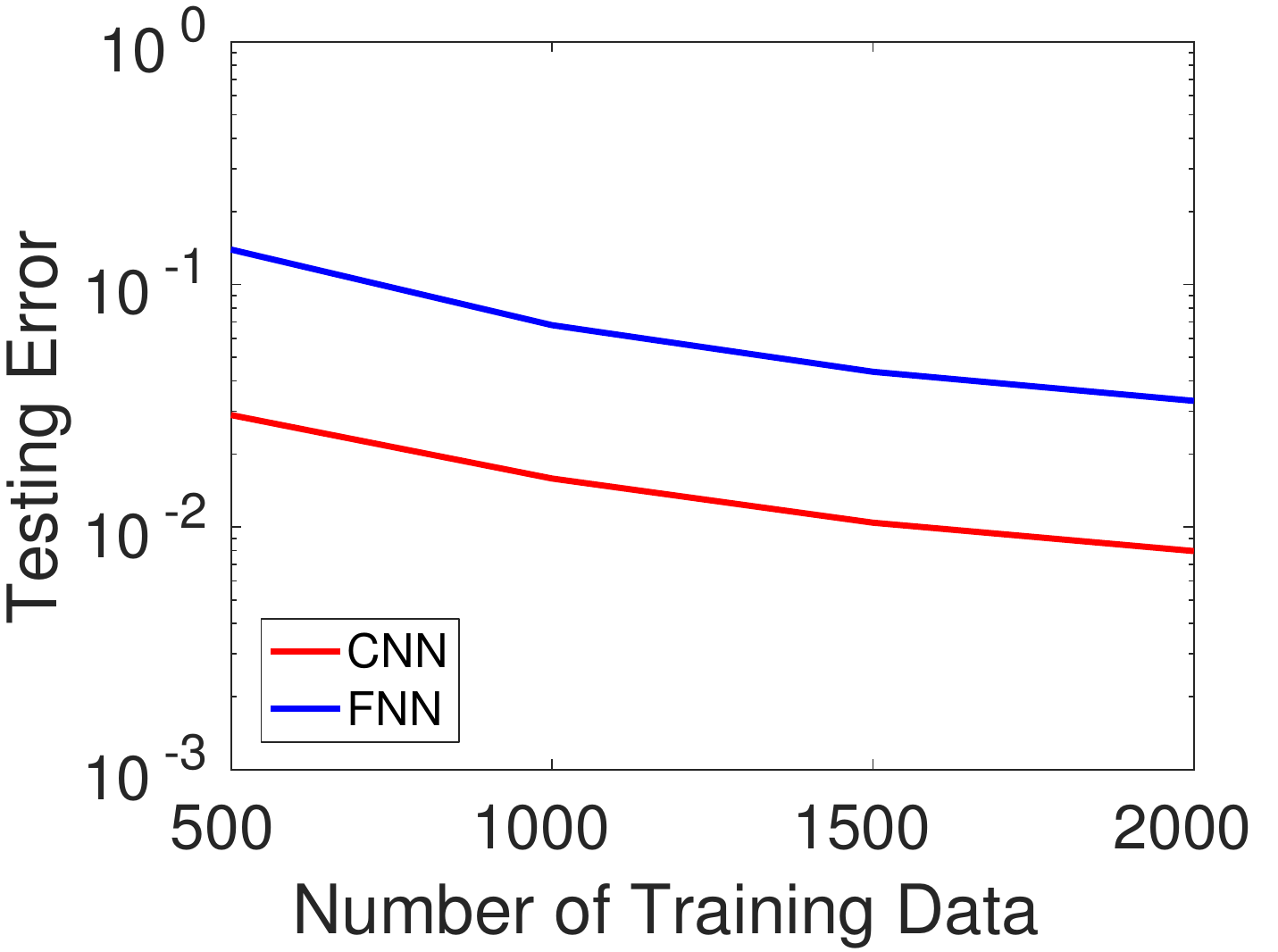}
		\caption{Filter size $m=16$.}
	\end{subfigure}	
	\caption{Experiments on the problem of estimating a convolutional filter with average pooling described in Eq.~\eqref{eqn:conv_filter_avg_pooling} with stride size $s=1$.
	}
	\label{fig:filter_stride1}
\end{figure*}

\begin{figure*}[t!]
	\centering
	\begin{subfigure}[t]{0.29\textwidth}
		\includegraphics[width=\textwidth]{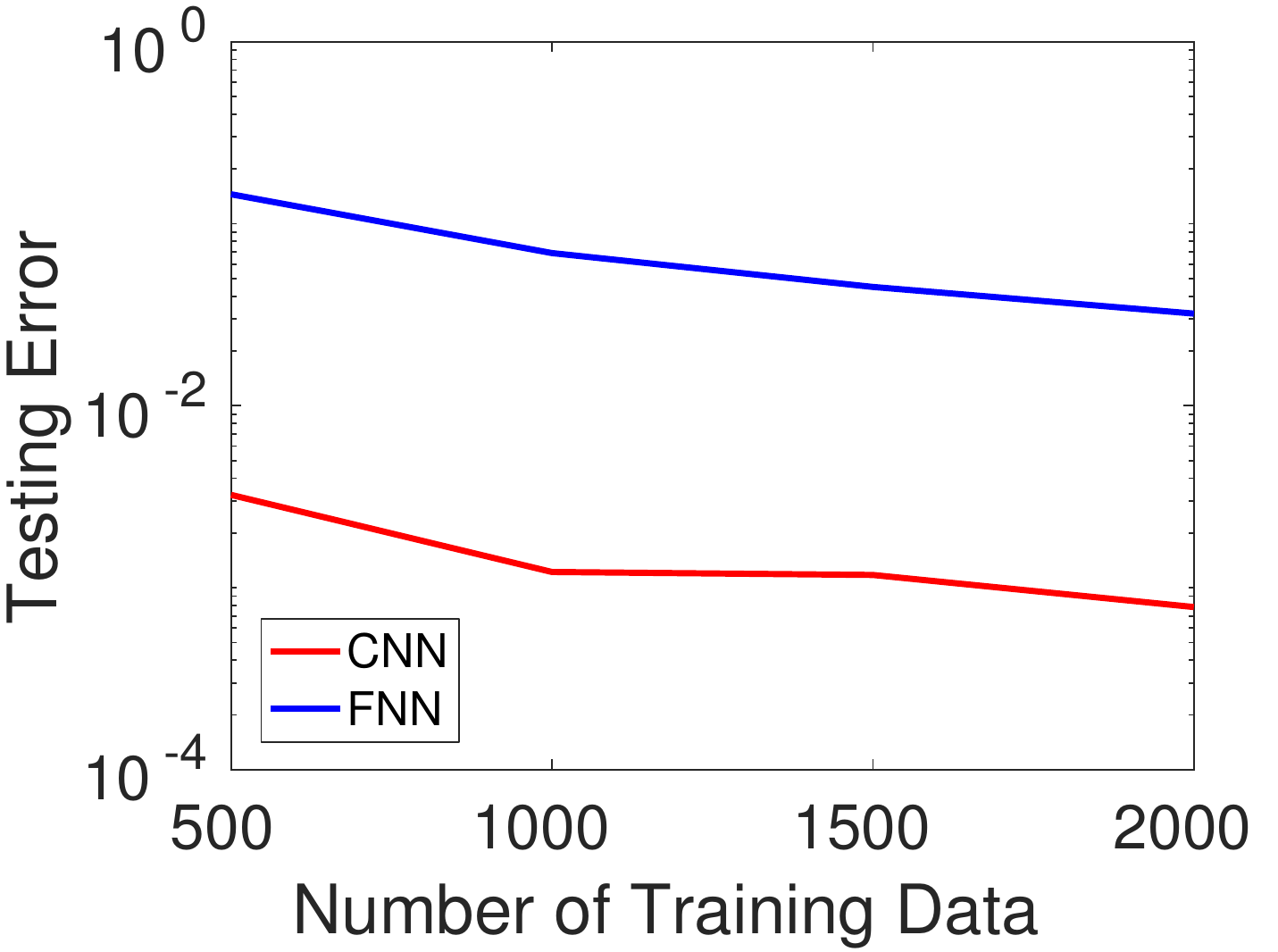}
		\caption{Filter size $m=2$.}
	\end{subfigure}	
	\quad
	\begin{subfigure}[t]{0.29\textwidth}
		\includegraphics[width=\textwidth]{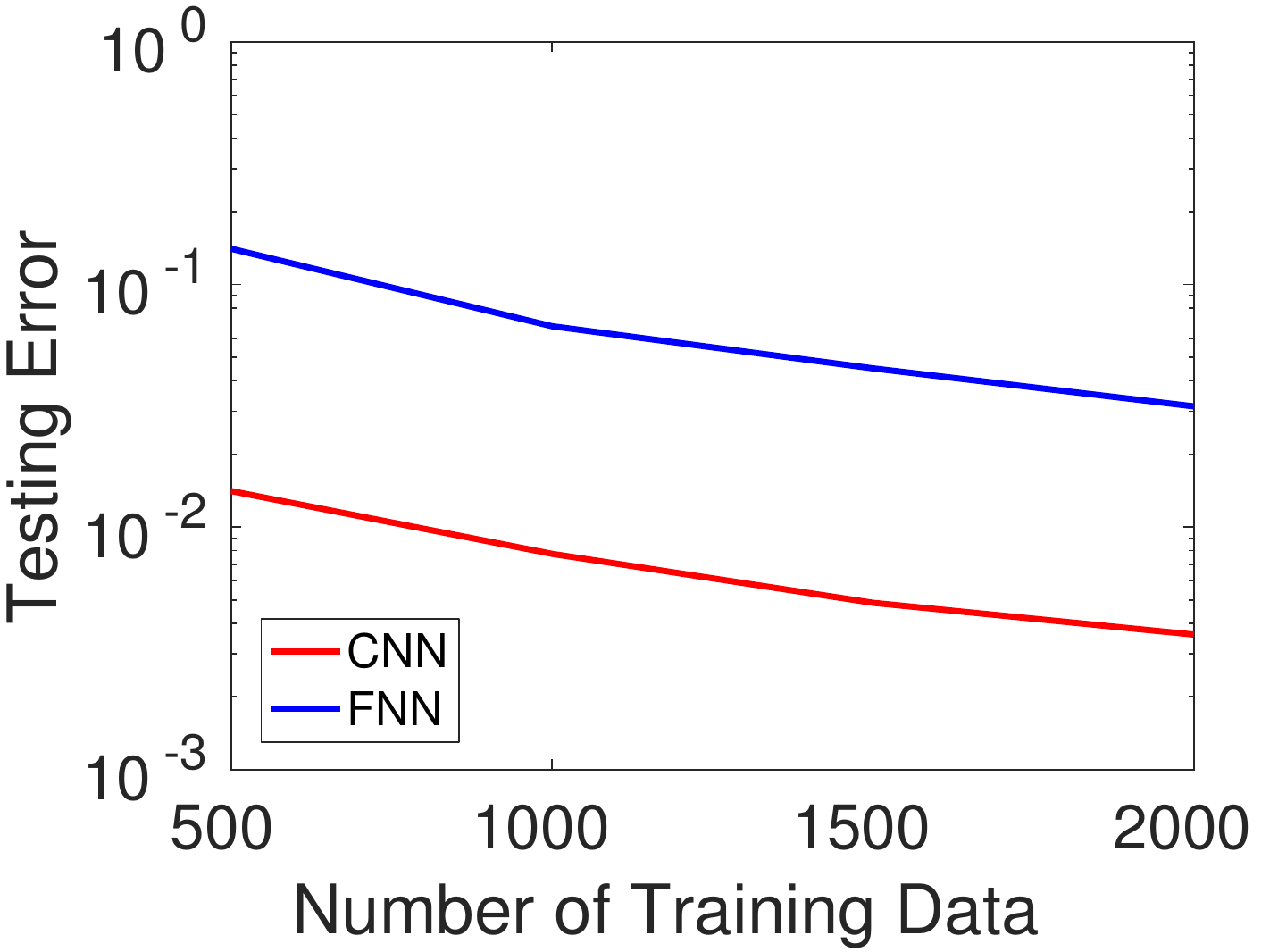}
		\caption{Filter size $m=8$.}
	\end{subfigure}	
	\quad
	\begin{subfigure}[t]{0.29\textwidth}
		\includegraphics[width=\textwidth]{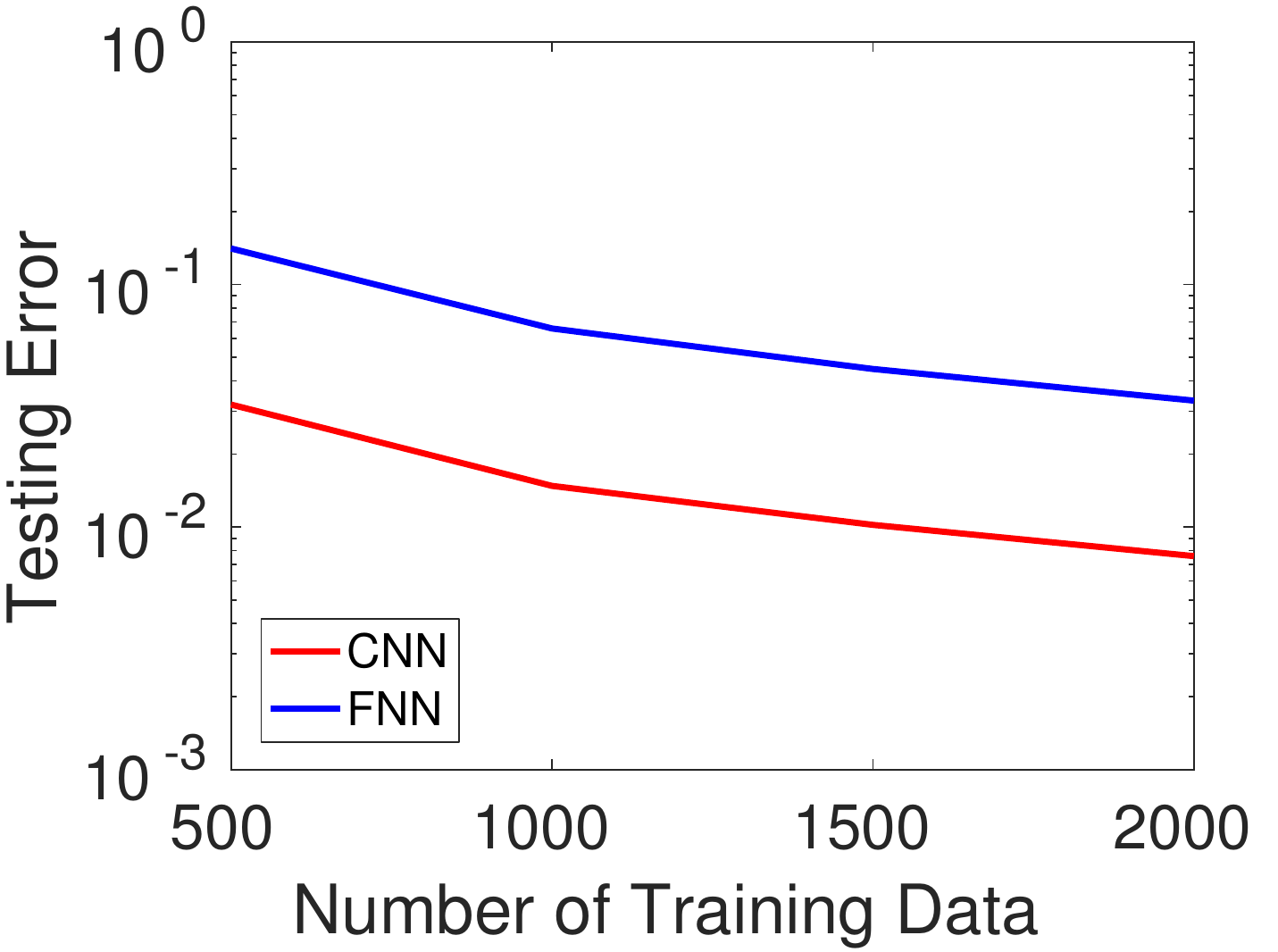}
		\caption{Filter size $m=16$.}
	\end{subfigure}	
	\caption{Experiments on the problem of estimating a convolutional filter with average pooling described in Eq.~\eqref{eqn:conv_filter_avg_pooling} with stride size $s=m$, i.e., non-overlapping.
	}
	\label{fig:filter_nonoverlap}
\end{figure*}

In Figure~\ref{fig:filter_stride1} and Figure~\ref{fig:filter_nonoverlap}, we consider the problem of estimating a convolutional filter with average pooling.
We vary the number of samples, the dimension of filters and the stride size.
Here we compare parameterizing the prediction function as a $d$-dimensional linear predictor and as a convolutional filter followed by average pooling.
Experiments show CNN parameterization is consistently better than the FNN parameterization. 
Further, as number of training samples increases, the prediction error goes down and as the dimension of filter increases, the error goes up.
These facts qualitatively justify our derived error bound $\widetilde{O}\left(\frac{m}{n}\right)$.
Lastly, in Figure~\ref{fig:filter_stride1} we choose stride $s=1$ and in Figure~\ref{fig:filter_nonoverlap} we choose stride size equals to the filter size $s=m$, i.e., non-overlapping.
Our experiment shows the stride does \emph{not} affect the prediction error in this setting which coincides our theoretical bound in which there is no stride size factor.

In Figure~\ref{fig:two_layer}, we consider the one-hidden-layer CNN model $F^{\cw}$.
Here we fix the filter size $m=8$ and vary the number of training samples and the stride size.
When stride $s=1$, convolutional parameterization has the same order parameters  as the linear predictor parameterization ($r=57$ so $r+m = 65 \approx d = 64$) and Figure~\ref{fig:two_layer_s1} shows they have similar performances.
In Figure~\ref{fig:two_layer_s4} and Figure~\ref{fig:two_layer_s8} we choose the stride to be $m/2 = 4$ and $m = 8$ (non-overlapping), respectively.
Note these settings have less parameters ($r+m = 23$ for $s=4$ and $r+m = 16$ for $s=8$) than the case when $s=1$ and so CNN gives better performance than FNN.

We conduct similar experiments to compare RNN and FNN in Figure~\ref{fig:rnn}.
We set input dimension $d=50$ and length of the sequence $L=50$.
Again we use Gaussian input and vary the number of hidden units and number of training data.
From Figure~\ref{fig:rnn}, it is clear that RNN parameterization requires much fewer samples than the naive FNN parameterization.

\begin{figure*}[t!]
	\centering
	\begin{subfigure}[t]{0.29\textwidth}
		\includegraphics[width=\textwidth]{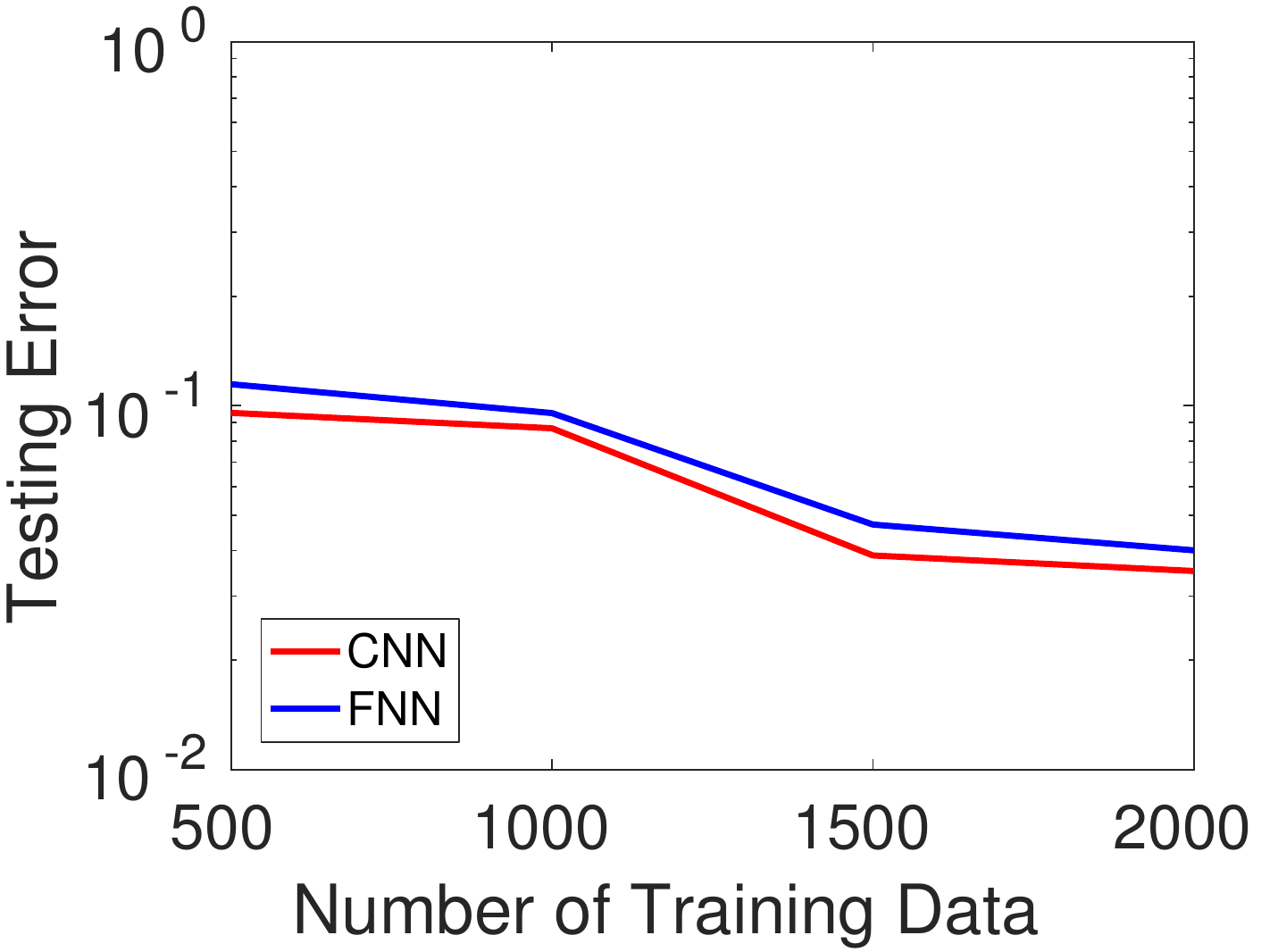}
		\caption{Stride size $s=1$.}
		\label{fig:two_layer_s1}
	\end{subfigure}	
	\quad
	\begin{subfigure}[t]{0.29\textwidth}
		\includegraphics[width=\textwidth]{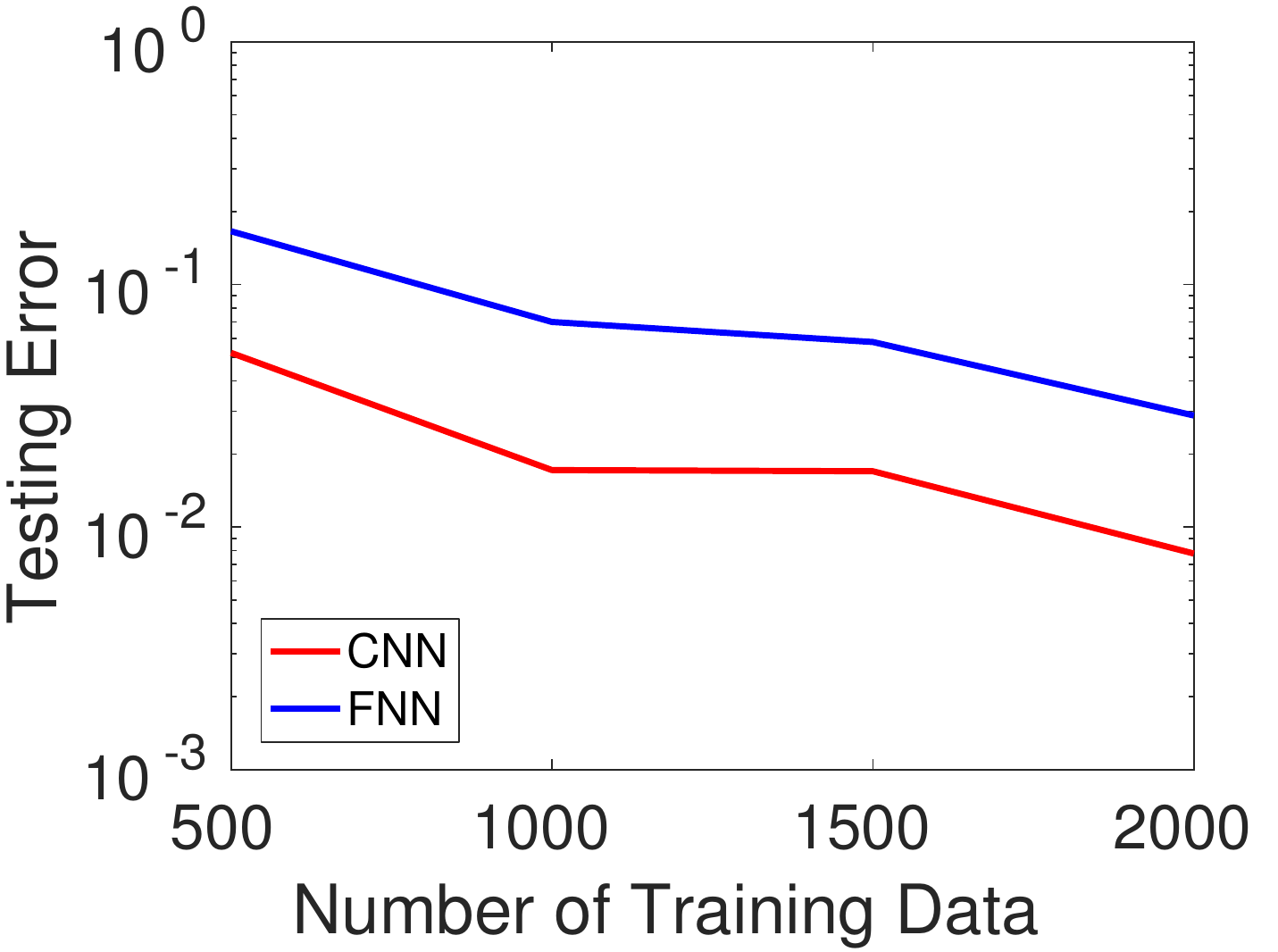}
		\caption{Stride size $s=m/2$.}
		\label{fig:two_layer_s4}
	\end{subfigure}	
	\quad
	\begin{subfigure}[t]{0.29\textwidth}
		\includegraphics[width=\textwidth]{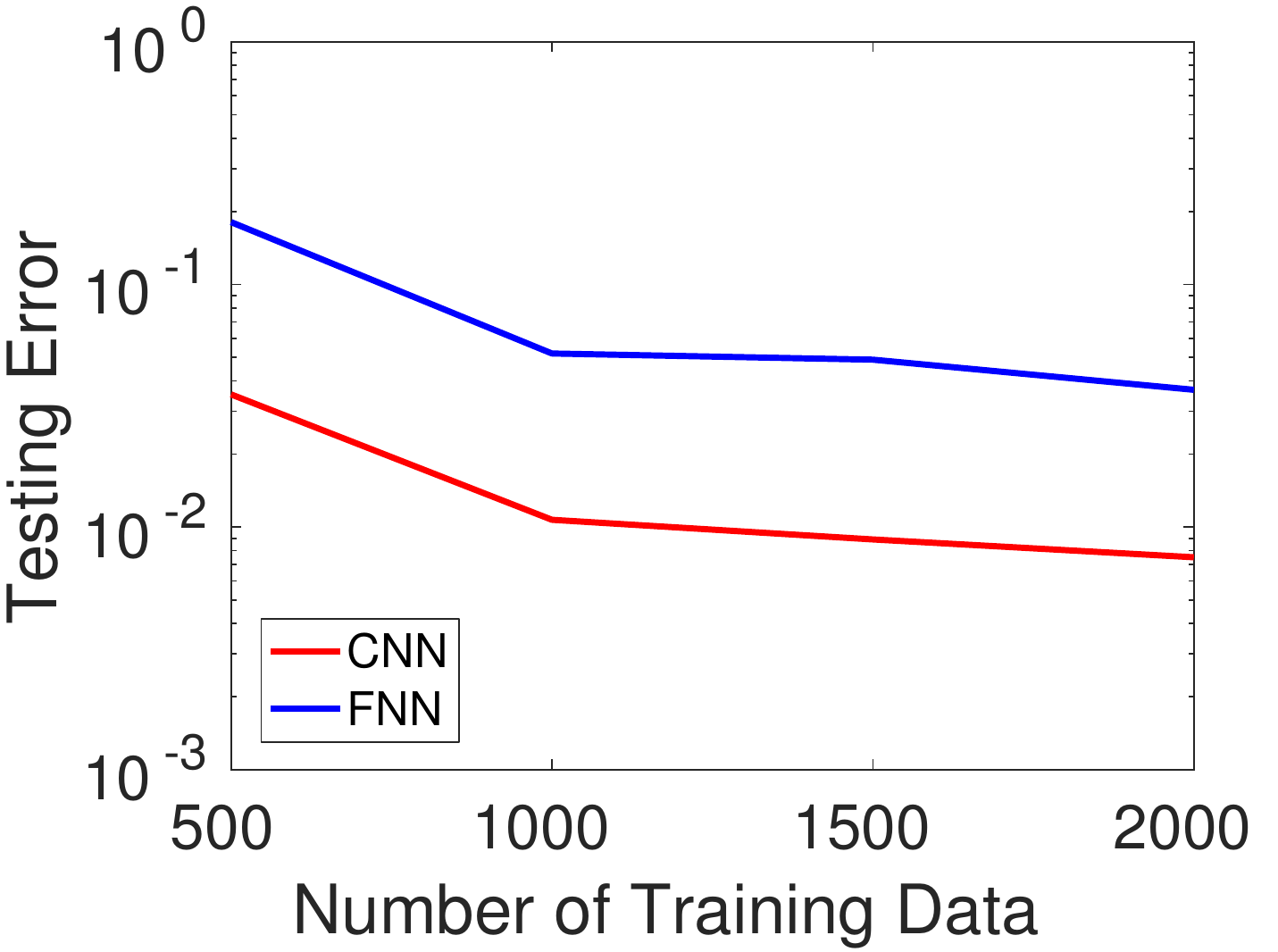}
		\caption{Stride size $s=m$, i.e., non-overlapping.}
		\label{fig:two_layer_s8}
	\end{subfigure}	
	\caption{Experiment on the problem of one-hidden-layer convolutional neural network with a shared filter and a prediction layer described in Eq.~\eqref{eqn:two_layer}. The filter size $m$ is chosen to be $8$.
	}
	\label{fig:two_layer}
\end{figure*}

\begin{figure*}[t!]
	\centering
	\begin{subfigure}[t]{0.29\textwidth}
		\includegraphics[width=\textwidth]{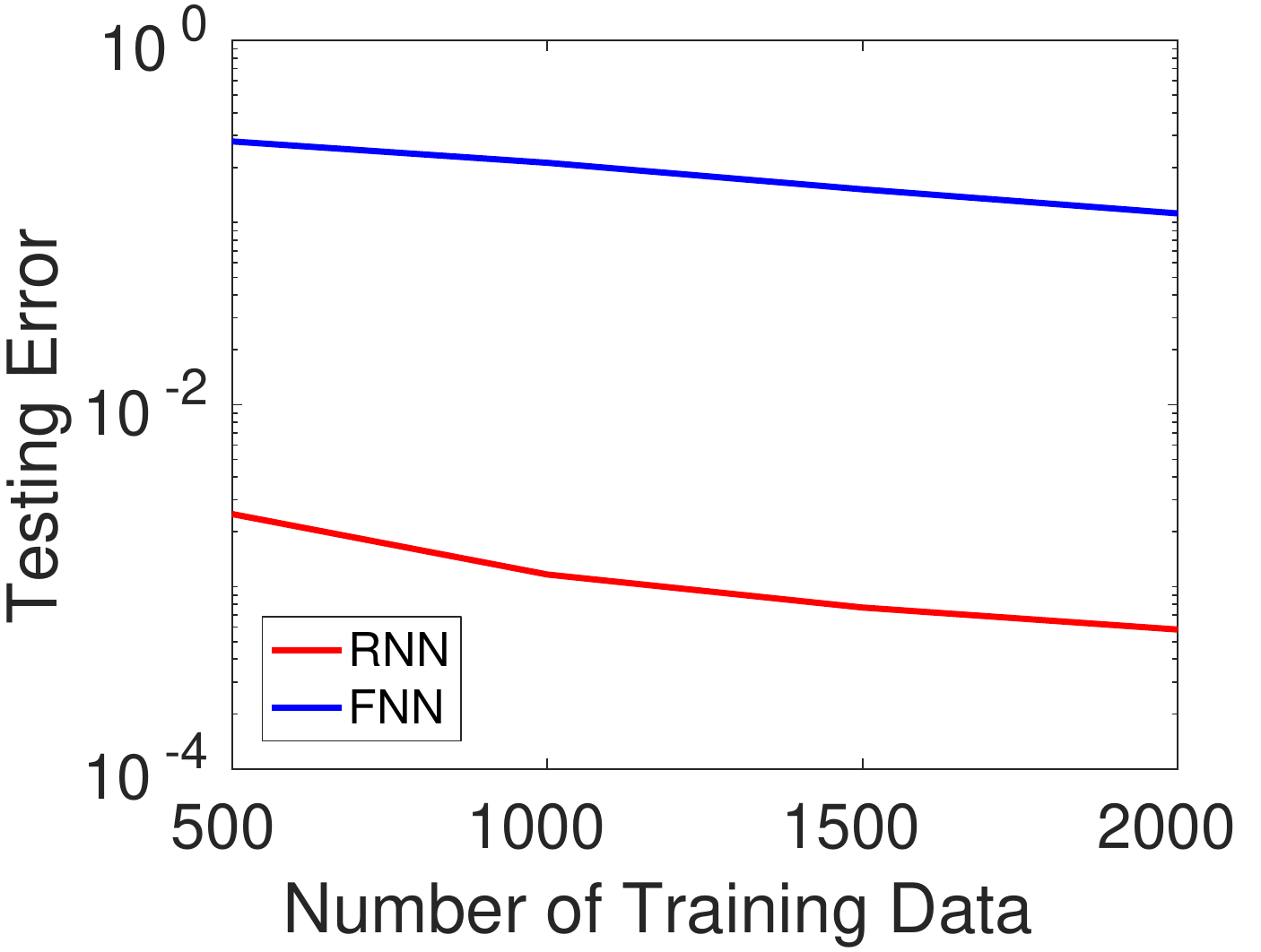}
		\caption{Hidden units $r=2$.}
		\label{fig:rnn_r2}
	\end{subfigure}	
	\quad
	\begin{subfigure}[t]{0.29\textwidth}
		\includegraphics[width=\textwidth]{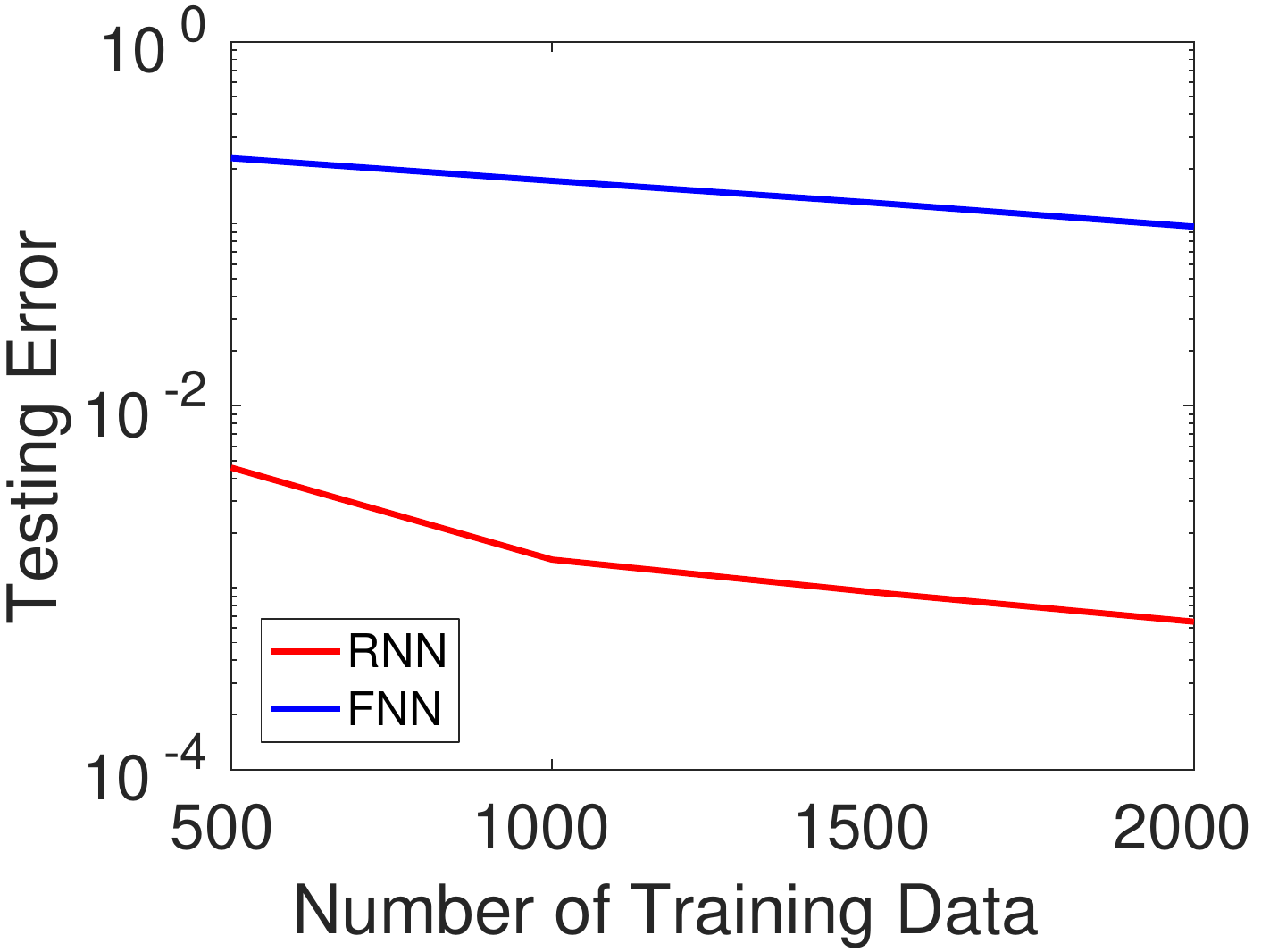}
		\caption{Hidden units $r=8$.}
		\label{fig:rnn_r8}
	\end{subfigure}	
	\quad
	\begin{subfigure}[t]{0.29\textwidth}
		\includegraphics[width=\textwidth]{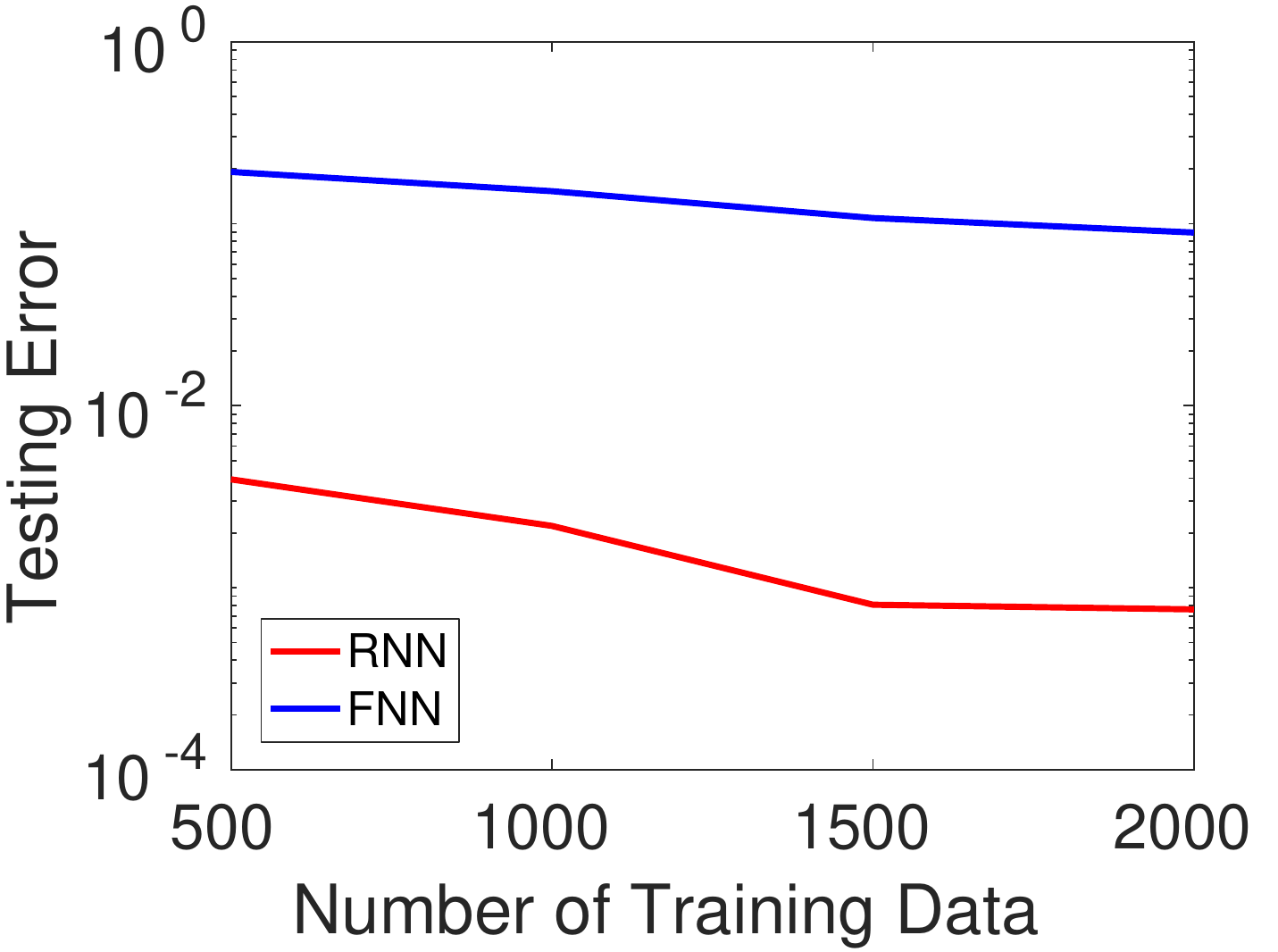}
		\caption{Hidden units $r=16$.}
		\label{fig:rnn_r16}
	\end{subfigure}	
	\caption{Experiment on RNN described in Eq.~\eqref{eq:rnn-model}, we choose $d=L=50$ and varies the number of hidden units $r$ and number of training data.
	}
	\label{fig:rnn}
\end{figure*}

\section{Concluding remarks and discussion}
In this paper we give rigorous characterizations of the statistical efficiency of CNN with simple architectures.
Now we discuss how to extend our work to more complex models and  main difficulties.

\paragraph{Non-linear Activation:}
Our paper only considered CNN and RNN with linear activation. 
A natural question is what is the sample-complexity of estimating a CNN and RNN with non-linear activation like Recitifed Linear Units (ReLU).
We find that even without convolution structure,
this is a difficult problem.
For linear activation function, we can show the empirical loss is a good approximation to the population loss
and we used this property to derive our upper bound.
However, for ReLU activation, we can find a counter example for any finite $n$.
We believe if there is a better understanding of non-smooth activation,
we can extend our analysis framework to derive sharp sample-complexity bounds for CNN and RNN with non-linear activation function.

\paragraph{Multiple Filters:}  
For both CNN models we considered in this paper, there is only one shared filter.
In commonly used CNN architectures, there are multiple filters in each layer and multiple layers.
Note that if one considers a model of $k$ filters with linear activation with $k > 1$, one can always replace this model by a single convolutional filter that equals to the summation of these $k$ filters.
Thus, we can formally study the statistical behavior of wide and deep architectures only after we have understood the non-linear activation function.
Nevertheless, we believe our empirical process based analysis is still applicable.

\appendix

\section{Proofs of technical lemmas}

\subsection{Proof of Lemma~\ref{lem:star}}
For each of the parameter sets $\Theta$ we consider we need to verify that,
\begin{align*}
\sup_{\widehat{\theta} \in \Theta} \sum_{i=1}^n \xi_j \langle z^i, \hat\theta-\theta\rangle \leq \|\widehat{\theta} - \theta\|_X
\sup_{\phi \in \Theta_X} \sum_{i=1}^n \xi_j \langle z^i, \phi\rangle.
\end{align*}
It suffices to show for $\widehat{\theta} \in \Theta$ that the vector $v := (\widehat{\theta} - \theta)/\|\widehat{\theta} - \theta\|_X \in \Theta_X$. It is clear that the vector $\|v\|_X \leq 1$, and to complete the proof using the definition of the set $\Theta_X$ it is sufficient to show that 
$\theta/\|\widehat{\theta} - \theta\|_X,~\widehat{\theta}/\|\widehat{\theta} - \theta\|_X \in \Theta.$ Recall, the definitions:
\begin{enumerate}
\item For $F^{\cavg}$, $\Theta$ is the set of all $\theta := \sum_{\ell=0}^{r-1}\sS_s^\ell(w)$, where
$$
\sS_s^\ell(w) = [\underbrace{0, \ldots, 0}_{\ell s\;\; \text{zeros}}, w_{1}, \ldots, w_{m}, 0, \ldots, 0] \in \mathbb R^d;
$$
\item For $F^{\cw}$, $\Theta$ is the set of all $\theta := \sum_{\ell=0}^{r-1}a_\ell\sS_s^\ell(w)$;
\item For $F^{\rnn}$, $\Theta$ is the set of all $\theta := (\vct 1^\top A^{L-1}B~~ \;\;\vct 1^\top A^{L-2}B~~\; \;\ldots\;\; ~~\vct 1^\top B)$.
\end{enumerate}
In each case, given $\theta \in \Theta$ we can see that $c \theta \in \Theta$ for any $c \geq 0$. In more detail, in cases (1) and (2) we simply replace $w$ by $c w$ and in case (3) we replace $B$ by $c B$ to obtain a valid vector $c\theta \in \Theta$. As a consequence we see that, $\theta/\|\widehat{\theta} - \theta\|_X,~\widehat{\theta}/\|\widehat{\theta} - \theta\|_X \in \Theta,$ completing the proof.

\subsection{Proof of Lemma \ref{lem:re}}

Without loss of generality we only need to consider $\phi=\theta-\theta'$, $\theta,\theta'\in\Theta$, and $\|\phi\|_2=1$ (regardless of the value of $\rho$),
because the restricted eigenvalues are scale-invariant.
Also, we shall only prove the lower bound on $\lambda_{\min}$, with the upper bound on $\lambda_{\max}$
being a simple symmetric argument.

Let $\mathcal H\subseteq\overline\Theta_2(1)$ be the smallest set such that for any $\phi=\theta-\theta'$, $\|\phi\|_2=1$,
there exists $\phi'\in\mathcal H$ such that $\|\phi-\phi'\|_2\leq\epsilon$.
We then have for $\epsilon\in(0,1/2]$ that 
\begin{align}
\lambda_{\min}(\{z^i\}_{i=1}^n; \overline\Theta_2(\rho))
&\geq \inf_{\phi'\in\mathcal H}(\|\phi'\|_X-\|\phi-\phi'\|_X)^2 
\geq \inf_{\phi'\in\mathcal H}(\|\phi'\|_X - \max_i\|z^i\|_2 \|\phi-\phi'\|_2)^2\nonumber\\
&\geq \inf_{\phi'\in\mathcal H}(\|\phi'\|_X - O(Z\sqrt{\log(n/\delta)})\cdot \epsilon)^2 \label{eq:re-intermediate-1}\\
&\geq \inf_{\phi'\in\mathcal H}\|\phi'\|_X^2 - O(Z\epsilon\sqrt{\log(n/\delta)}),
\label{eq:re-1}
\end{align}
where Eq.~(\ref{eq:re-intermediate-1}) holds with probability $1-\delta/2$ because $\max_i\|z^i\|_2= O(Z\sqrt{\log(n/\delta)})$ with probability $1-\delta/2$,
by standard sub-Gaussian concentration inequalities.

In the rest of this proof we lower bound $\inf_{\phi'\in\mathcal H}\|\phi'\|_X^2$.
For every $\phi'\in\mathcal H$, there must exist $\phi\in\overline\Theta_2(1)$, $\|\phi\|_2=1$ such that $\|\phi'-\phi\|_2\leq \epsilon$,
because otherwise we can remove $\phi'$ from $\mathcal H$, which would violate the minimality of $\mathcal H$.
We then have $\|\phi'\|_2\geq \|\phi\|_2-\|\phi-\phi'\|_2 \geq 1-\epsilon\geq 1/2$, because $\epsilon\leq 1/2$ as assumed.
%
This implies that $\phi'\in\mathcal H$, $\|\phi'\|_\mu := \sqrt{\mathbb E_\mu[|\langle z,\phi'\rangle|^2]}\geq 1/2\sqrt{c}$
and $\|\phi'\|_\mu\leq 2\sqrt{C}$
for all $\phi'\in\mathcal H$, thanks to Assumption (A2).

Next fix arbitrary $\phi'\in\mathcal H$.
By sub-Gaussianity of $\{z^i\}$, with probability $1-\delta/2$ we have that $\max_i |\langle z^i,\phi'\rangle| \lesssim Z\sqrt{\log(n/\delta)}$ because $\|\phi'\|_2\leq 2$.
Using Hoeffding's inequality \citep{hoeffding1963probability} we have with probability $1-\delta/2$, conditioned on the event $\max_i |\langle z^i,\phi'\rangle| \lesssim Z\sqrt{\log(n/\delta)}$, that
\begin{align}
\big|\|\phi'\|_X^2-\|\phi'\|_\mu^2\big| 
&= \left|\frac{1}{n}\sum_{i=1}^n |\langle z^i,\phi'\rangle|^2 - \mathbb E_\mu[|\langle z,\phi'\rangle|^2]\right|\lesssim Z\sqrt{\log(n/\delta)}\cdot \sqrt{\frac{\log(1/\delta)}{n}}.\label{eq:re-3}
\end{align}

Lemma \ref{lem:re} is then proved, by combining Eqs.~(\ref{eq:re-1},\ref{eq:re-3}) and the established fact that $1/2\sqrt{c}\leq \|\phi'\|_\mu\leq 2\sqrt{C}$,
and using an union bound over all $\phi'\in\mathcal H$,
whose size is upper bounded by $|\mathcal H|\leq N(\epsilon;\overline\Theta_2(1),\|\cdot\|_2)$.

\subsection{Proof of Lemma \ref{lem:linear-subspace-cover}}

Without loss of generality we only prove Lemma \ref{lem:linear-subspace-cover} for the case of $\rho=1$,
while the general case of $\rho\neq 1$ is implied by multiplying $u,v$ and $\epsilon'$ by $\rho$ in Eq.~(\ref{eq:linear-subspace-cover}).

Let $\mathcal U,\mathcal V$ be two linear subspace of $\mathbb R^q$ of dimension at most $K$.
Let $U,V\in\mathbb R^{q\times k}$ be the corresponding orthonormal basis of $\mathcal U$ and $\mathcal V$, with orthogonal columns.
Any $u\in\mathcal U$, $\|u\|_2\leq 1$ can then be written as $u=U\alpha$ with $\|\alpha\|_2=1$.
Consider $v:=V\alpha$. It is easy to verify that $v\in\mathcal V$ and $\|v\|_2\leq 1$. In addition,
$\|u-v\|_2 = \|(U-V)\alpha\|_2 \leq \|U-V\|_\op \leq \|U-V\|_F$.
Subsequently, a covering of $\{U\in\mathbb R^{q\times k}: \|U\|_F\leq \sqrt{k}\|U\|_\op = \sqrt{k}\}$ in $\|\cdot\|_F$
up to precision $\epsilon'$ implies a covering in the sense of Eq.~(\ref{eq:linear-subspace-cover}).
By viewing $U$ as a $(k\times q)$-dimensional vector in the Euclidean space, it is easy to see that such a cover exists with size $\log N \lesssim (kq)\log(kq/\epsilon')\lesssim kq\log(q/\epsilon')$.

\subsection{Proof of Corollary \ref{cor:fano}}

Select $\mathfrak d(\theta_j,\theta_k) := \|\theta_j-\theta_k\|_2$.
Condition 1 in Lemma \ref{lem:fano} is clearly satisfied with $\rho=\rho_{\min}=\min_{j>0}\|\theta_0-\theta_j\|_2/2$.
Condition 2 in Lemma \ref{lem:fano} is also satisfied, because the noise variables $\{\xi_i\}_{i=1}^n$ follow Gaussian distributions whose support
span the entire $\mathbb R$.
For Condition 3, note that the KL divergence between $P_j,P_0$ parameterized by $\theta_j$ and $\theta_0$ can be computed as
\begin{align*}
\kl(P_j\|P_0) &= \sum_{i=1}^n \mathbb E_{x_i}\left[ \kl(\mathcal N(x_i^\top\theta_j, \sigma^2)\|\mathcal N(x_i^\top\theta_0,\sigma^2))\right]
= \sum_{i=1}^n \mathbb E_{x_i}\left[\frac{|x_i^\top(\theta_j-\theta_0)|^2}{2\sigma^2}\right]\\
&= \frac{n\|\theta_j-\theta_0\|_2^2}{2\sigma^2},
\end{align*}
where the last equality holds because $\{x_i\}_{i=1}^n \overset{i.i.d.}{\sim} \mathcal N_d(0, I)$.
Hence,
$$
\frac{1}{M}\sum_{j=1}^M\kl(P_{j}\|P_0) = \frac{n}{2\sigma^2}\times \frac{1}{M}\sum_{j=1}^M \|\theta_j-\theta_0\|_2^2 = \frac{n}{2\sigma^2}\times \rho_{\avg}^2, 
$$
and therefore $\gamma$ in Condition 3 of Lemma \ref{lem:fano} can be chosen as $\gamma=(n\rho_{\avg}^2)/(2\sigma^2\log M)$.

\subsection{Proof of Lemma \ref{lem:free-vary}}

Without loss of generality assume $\mathcal I$ corresponds to the first $|\mathcal I|$ components of $\theta\in\mathbb R^D$.
The first step is to construct a finite set of binary vectors $\mathcal H\subseteq\{0,1\}^{|\mathcal I|}$ such that 
\begin{equation}
\forall h,h'\in\mathcal H, \;\;\;\; \Delta_H(h,h') = \sum_{i=1}^{|\mathcal I|} \vct 1\{h_i\neq h_i'\} \gtrsim |\mathcal I|.
\label{eq:cw-code}
\end{equation}

Using the construction of constant-weight codes (e.g., \citep[Lemma 9]{wang2016noise}, \citep[Theorem 7]{graham1980lower}), 
a finite set $\mathcal H$ satisfying Eq.~(\ref{eq:cw-code}) exists, with size lower bounded by $\log|\mathcal H|\gtrsim |\mathcal I|$.

Next define
$$
\Theta' := \{\theta\in\Theta: \theta(\mathcal I)=\epsilon h, h\in\mathcal H\}.
$$
The existence of such a $\Theta'$ is guaranteed by the condition of this lemma, where $\epsilon>0$ is a small positive number to be specified later.
It is easy to see that $\Theta'$ and $\mathcal H$ has one-to-one correspondence, and therefore $\log|\Theta'|=\log|\mathcal H|\gtrsim |\mathcal I|$.
Furthermore, for any $\theta,\theta'\in\Theta'$ and their corresponding $h,h'\in\mathcal H$, it holds that
\begin{align*}
\|\theta-\theta'\|_2 &\geq \|\theta(\mathcal I)-\theta'(\mathcal I)\|_2 = \epsilon\times \sqrt{\Delta_H(h,h')} \gtrsim \epsilon\sqrt{|\mathcal I|};\\
\|\theta-\theta'\|_2^2 &\geq \|\theta(\mathcal I)-\theta'(\mathcal I)\|_2^2 = \epsilon^2\times \Delta_H(h,h')\gtrsim \epsilon^2|\mathcal I|.
\end{align*}
Consequently, $\rho_{\min}\gtrsim \epsilon|\mathcal I|$ and $\rho_{\avg}^2\gtrsim \epsilon^2|\mathcal I|$.
Setting $\epsilon\asymp \sigma/\sqrt{n}$ and invoking Corollary \ref{cor:fano} we complete the proof of Lemma \ref{lem:free-vary}.

\bibliography{refs}

\end{document}